\documentclass{article}
\usepackage{ijcai17}

\usepackage{times}              \usepackage{latexsym}
\usepackage{amsmath}
\usepackage{amsfonts}
\usepackage{amsthm}
\usepackage{epsfig}
\usepackage{url}
\usepackage{xcolor,colortbl}
\usepackage{epstopdf}
\usepackage{bbm,bm}
\usepackage{xspace}

\usepackage[utf8]{inputenc} \usepackage[T1]{fontenc}    \usepackage{hyperref}       \usepackage{url}            \usepackage{booktabs}       \usepackage{amsfonts}       \usepackage{nicefrac}       \usepackage{microtype}      \usepackage{graphicx} \usepackage{multirow}
\usepackage{hhline}
\usepackage{color}

\usepackage{algorithm}
\usepackage{algorithmic}

\usepackage{graphicx}

\makeatletter
\@ifundefined{showcaptionsetup}{}{ \PassOptionsToPackage{caption=false}{subfig}}
\usepackage{subfig}
\DeclareCaptionType{noticebox}

\usepackage{wrapfig}

\usepackage{setspace}

\floatstyle{ruled}
\newfloat{algorithm}{tbp}{loa}
\providecommand{\algorithmname}{\small Algorithm}
\floatname{algorithm}{\protect\algorithmname}

\newcommand{\colcut}{\hspace{-0.3em}}
\newcommand{\colcutb}{\hspace{-0.4em}}
 \theoremstyle{plain}
\newtheorem{theorem}{Theorem}[section]
\newtheorem{lemma}[theorem]{Lemma}
\newtheorem{corollary}[theorem]{Corollary}

\theoremstyle{definition}

\newtheorem{prop}{Proposition}
\theoremstyle{remark}

\newcommand{\R}{\mathbb{R}}

\newcommand{\argmin}{\arg\!\min}
\newcommand{\relu}{{\rm ReLU}}
\newcommand{\vct}[1]{\bm{#1}}
\newcommand{\mtx}[1]{\bm{#1}}

\newcommand{\E}{\mathbb E}
\newcommand{\Prob}{\mathbb P}

\def\vctx{\bm x} \def\vcthatx{\hat{\bm x}} \def\vctz{\bm z} \def\vcthid{\bm h} \def\vcthatz{\hat{\bm z}} 
\newcommand\citet[1]{\citeauthor{#1}~[\citeyear{#1}]}
\newcommand\citep[1]{\cite{#1}}

\newenvironment{reftheorem}[1]{\medskip\parindent 0pt{\bf Theorem \ref{#1}}\em }{\vspace{1em}}

\usepackage{color}
\definecolor{gray}{rgb}{0.5,0.5,0.5}

\iftrue  
            \newcommand{\todo}[1]{}
        \newcommand{\outline}[1]{}
        \newcommand{\textgray}[1]{}
        \newcommand{\commenttext}[1]{}
        \newcommand{\commentfoot}[1]{}
        \newcommand{\commentselfoot}[2]{}
        \newcommand{\commentselrep}[2]{}
        \newcommand{\topic}[1]{}
        \newcommand{\commenthl}[1]{}
\else 
            \newcommand{\todo}[1]{{\textcolor{red}{[[TODO: {#1}]]}}}
        \newcommand{\outline}[1]{{\textcolor{blue}{[[{#1}]]}}}
        \newcommand{\textgray}[1]{\textcolor{gray}{[[{#1}]]}}
        \newcommand{\commenttext}[1]{\textcolor{red}{[[{#1}]]}}
        \newcommand{\commentfoot}[1]{\footnote{\textcolor{red}{\textit{#1}}}}
        \newcommand{\commentselfoot}[2]{{\textcolor{blue}{#1}}\commenttext{#2}}
        \newcommand{\commentselrep}[2] {{\textcolor{blue}{#1}} {\textcolor{green}{[[\textit{#2}]]}}}
        \newcommand{\topic}[1]{\textcolor{gray}{\textbf{(#1.)}}}
        \newcommand{\commenthl}[1]{\textcolor{blue}{[HL: #1]}}
\fi

\usepackage[font={small}]{caption}

\newcommand{\cutsectionup}{}
\newcommand{\cutsectiondown}{}
\newcommand{\cutsubsectionup}{}
\newcommand{\cutsubsectiondown}{}

\newcommand{\mytitle}{Towards Understanding the Invertibility of  Convolutional Neural Networks}
\title{\mytitle}

\iftrue
\author{
	Anna C.~Gilbert$^1$ \quad Yi Zhang$^1$ \quad Kibok Lee$^1$ \quad Yuting Zhang$^1$ \quad Honglak Lee$^{1,2}$\\
	$^1$University of Michigan, Ann Arbor, MI 48109 \\
	$^2$Google Brain, Mountain View, CA 94043\\
	\texttt{\{annacg,yeezhang,kibok,yutingzh\}@umich.edu} \\ 
	\texttt{honglak@\{umich.edu,google.com\}} \\
}
\else
\author{\normalfont Paper 1906}
\fi

\begin{document}

\maketitle

\begin{abstract} 
Several recent works have empirically observed that Convolutional Neural Nets (CNNs) are (approximately) invertible. 
To understand this approximate invertibility phenomenon and how to leverage it more effectively, we focus on a theoretical explanation and develop a mathematical model of sparse signal recovery that is consistent with CNNs with random weights.
We give an exact connection to a particular model of model-based compressive sensing (and its recovery algorithms) and random-weight CNNs.
We show empirically that several learned networks are consistent with our mathematical analysis and then demonstrate that with such a simple theoretical framework, we can obtain reasonable reconstruction results on real images. 
We also discuss gaps between our model assumptions and the CNN trained for classification in practical scenarios. 
\vspace*{-0.1in}
\end{abstract} 
\cutsectionup
\section{Introduction}
\cutsectiondown

Deep learning has achieved remarkable success in many technological areas, including automatic speech recognition~\citep{hinton2012deep,hannun2014deep}, natural language processing~\citep{collobert2011natural,mikolov2013distributed,cho2014learning}, and computer vision, in particular with deep Convolutional Neural Networks (CNNs)~\citep{lecun1989backpropagation,alexnet,vggnet,szegedy2015going}.

Following the unprecedented success of deep networks, there have been some theoretical works~\citep{Arora:2014vi,Arora:2015vt,Paul:2014vb} that suggest several mathematical models for different deep learning architectures.
However, theoretical analysis and understanding lag behind the very rapid evolution and empirical success of deep architectures, and 
more theoretical analysis is needed to better understand the state-of-the-art deep architectures, and possibly to improve them further.

In this paper, we address the gap between the empirical success and theoretical understanding of the CNNs, in particular its invertibility (i.e., reconstructing the input from the hidden activations), by analyzing a simplified mathematical model using random weights (See Section~\ref{sec:relatedwork-randomfilters} and \ref{sec:exp-randomfilters} for the practical relevance of the assumption).

This property is intriguing because CNNs are typically trained with discriminative objectives (i.e., unrelated to reconstruction) with a large amount of labels, such as the ImageNet dataset~\citep{imagenet}.
\citet{recover-lp-pooling} studied signal discovery from generalized pooling operators using image patches on non-convolutional small scale networks and datasets.
\citet{invert-cnn} used upsampling-deconvolutional architectures to invert the hidden activations of feedforward CNNs to the input domain. 
In another related work, \citet{what-where} proposed a stacked what-where autoencoder network and demonstrated its promise in unsupervised and semi-supervised settings. 
\citet{deconv-recon} showed that CNNs discriminately trained for image classification (e.g., VGGNet~\citep{vggnet}) are almost fully invertible using pooling switches.
Despite these interesting results, there is no clear theoretical explanation as to why CNNs are invertible yet.

\textgray{
Our setup is also similar to recent work of Zhang et al. (http://bit.ly/2a6sR) which empirically shows that CNNs are almost fully invertible using pooling switches (it also shows that encouraging this property improves large-scale image classification).
}

We introduce three new concepts that, coupled with the accepted notion that images have sparse representations, guide our understanding of CNNs:
\begin{enumerate}
		\item we provide a \emph{particular} model of sparse linear combinations of the learned filters that are consistent with natural images; also, this model of sparsity is itself consistent with the feedforward network;
	\item we show that the effective matrices that capture explicitly the convolution of multiple filters exhibit a model-Restricted Isometry Property (model-RIP)~\citep{Baraniuk:2010hg}; and
	\item our model can explain each layer of the feedforward CNN algorithm as one iteration of Iterative Hard Thresholding (IHT)~\citep{Blumensath:2009tm} for model-based compressive sensing and, hence, we can reconstruct the input simply and accurately.
			\end{enumerate}
In other words, we give a theoretical connection to a particular version of model-based compressive sensing (and its recovery algorithms) and CNNs.
Using the connection, we give a reconstruction bound for a single layer in CNNs, which can be possibly extended to multiple layers.
In the experimental sections, we show empirically that large-scale CNNs are consistent with our mathematical analysis.
This paper explores these properties and elucidates specific empirical aspects that further mathematical models might need to take into account.

\textgray{Gaussian random weights are not practically relevant: Basically, this is described in Section~\ref{sec:model-rip}, model-RIP and random filters. First few sentences are added as footnote. cifar-10 experiment would be added in either experiment section or supplementary material.\\
We note that a model should not be an exact replica of a real setting; it should be a simplified but representative abstraction of practical settings. A number of work show that random weight CNNs still achieve surprisingly good classification accuracy although they may not match the state-of-the-art results. For example, see: Saxe et al. (http://bit.ly/2a1jOM), Jarret et al. (http://bit.ly/2a2iDG), Pinto et al. (http://bit.ly/2a6siG). We also tried random filters to build a CNN of three conv-pooling layers, and train a softmax classifier upon it for Cifar-10. It gives reasonable accuracy (~$74\%$), demonstrating the practical relevance of our assumption. This is not only about the random initialization of networks. Moreover, while the learned filters do not necessarily follow a Gaussian distribution, the model-RIP bound still holds in practice as we experimentally showed.
}

 \cutsectionup
\section{Preliminaries}
\cutsectiondown

In this section, we begin with discussion on the effectiveness of random weights in CNNs, and then we provide the notations for CNNs, compressive sensing, and sparse signal recovery.

\textgray{Single vs multi layers, 1D vs 2D\\We focus on a single layer in the analysis instead of multiple layers. We can extend the equivalency on a single layer of CNNs to multiple layer CNNs simply by using the output on one layer as the input to another, still using the steps of the inner loop of IHT. Similarly, we can easily generalize our analysis of convolution in one dimension to multiple dimensions and, in fact, our experiments are carried out in two dimensions on real images.}

\textgray{Our model allows stride>1.\\Not all filters in each dashed box in Fig.1 are identical; We take K filters (output channels) into account. -> this is clearly described in the caption of fig 1}

\cutsubsectionup
\subsection{Effectiveness of Gaussian Random Filters}
\label{sec:relatedwork-randomfilters}
\cutsubsectiondown

CNNs with Gaussian random filters have been shown to be surprisingly effective in unsupervised and supervised deep learning tasks.
\citet{jarrett2009best} showed that random filters in 2-layer CNNs work well for image classification. Also, ~\citet{saxe2011random} observed that convolutional layer followed by pooling layer is frequency selective and translation invariant, even with random filters, and these properties lead to good performance for object recognition tasks.
On the other hand, ~\citet{giryes2015deep} proved that CNNs with random Gaussian filters have metric preservation property, and they argued that the role of training is to select better hyperplanes discriminating classes by distorting boundary points among classes.
According to their observation, random filters are in fact a good choice if training data are initially well-separated.
Also, ~\citet{he2016powerful} empirically showed that random weight CNNs can do image reconstruction well.

\textgray{In particular, ~\citet{Bengio:2010ud} demonstrated that the success of training deep neural networks largely depends on random initialization.}
\commenthl{Explain each work in more detail---work by work, and put more related work.}
\commenthl{I don't see the main rationale about talking about importance of random initializations.
Random initialization is always used by default, and it doesn't support why random weights (no training) are reasonable.
I think it's better to focus on random filter networks, and avoid talking about non-random filter networks.}
\textgray{~\citet{giryes2015deep} analyzed the role of initialization and training for CNNs by calculating the histogram of distances among representations of different examples.
Specifically, they showed that, with Gaussian random initialization, the training algorithm hardly changes the histogram on MNIST and CIFAR-10 datasets.}
\commenthl{Is the above property about random weights or learned weights?}

To better demonstrate the effectiveness of Gaussian random CNNs, we evaluate their classification performance on CIFAR-10 \citep{cifar10} in Section~\ref{sec:exp-randomfilters}.
Although the performance is not the state-of-the-art, it is surprisingly good considering that the networks are almost untrained.
Our theoretical results may provide a new perspective on explaining these phenomena. 

\cutsubsectionup
\subsection{Convolutional Neural Nets}
\label{sec:cnn}
\cutsubsectiondown

\begin{figure}\centering
\includegraphics[width=.43\textwidth]{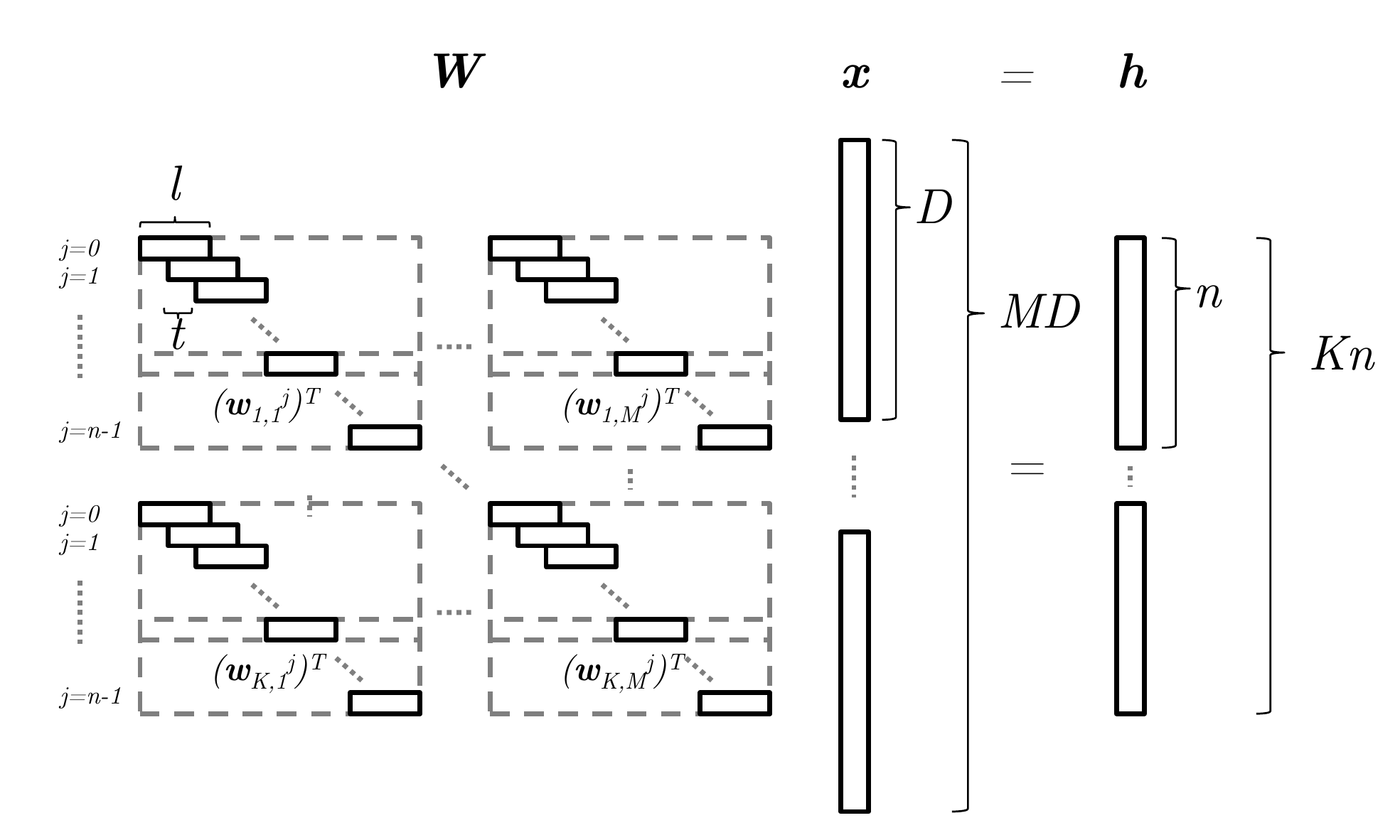}
\caption{One-dimensional CNN architecture where $\mtx{W} \in \R^{Kn \times MD}$ is the matrix instantiation of convolution over $M$ channels with a filter bank consisting of $K$ different filters. Note that a filter bank has K filters of size $l\times M$, such that there are $lMK$ parameters in this architecture.}
\label{fig:1d-cnn-arch}
\end{figure}

For simplicity, we vectorize input signals to 1-d signal;
for any operations we would ordinarily carry out on images, we do on vectors with the appropriate modifications.
We define a single layer of our CNN as follows.
We assume that the input signal $\vctx$ consists of $M$ channels, each of length $D$, and we write $\vctx \in \R^{MD}$.
For each of the input channels, $m = 1,\ldots,M$, let $\vct{w}_{i,m}$, $i = 1,\ldots, K$ denote one of $K$ filters, each of length $\ell$.
Let $t$ be the stride length, the number of indices by which we shift each filter.
Note that $t$ can be larger than 1.
We assume that the number of shifts, $n = (D - \ell)/t + 1$, is an integer.
Let $\vct{w}_{i,m}^j$ be a vector of length $D$ that consists of the $(i,m)$-th filter shifted by $jt$, $j = 0,\ldots,n-1$ (i.e., $\vct{w}_{i,m}^j$ has at most $\ell$ non-zero entries).
We concatenate over the $M$ channels each of these vectors (as row vectors) to form a large matrix $\mtx{W}$, which is the $Kn \times MD$ matrix made up of $K$ blocks of the $n$ shifts of each filter in each of $M$ channels.
We assume that $Kn \geq MD$ and the $Kn$ row vectors of $\mtx{W}$ span $\R^{MD}$ and that we have normalized the rows so that they have unit $\ell_2$ norm.
The hidden units of the feed-forward CNN are computed by multiplying an input signal $\vctx \in \R^{MD}$ by the matrix $\mtx{W}$ (i.e., convolving, in each channel, by a filter bank of size $K$, and summing over the channels to obtain $Kn$ outputs).\footnote{Convolution can be computed more efficiently than matrix multiplication, but they are mathematically equivalent.}
We use $\vcthid = \mtx{W} \vctx$ for the hidden activation computed by a single layer CNN without pooling.
Figure~\ref{fig:1d-cnn-arch} illustrates the architecture. 
As a nonlinear activation, we apply the $\relu$ function to the $Kn$ outputs, and then selecting the value with maximum absolute value in each of the $K$ blocks;
i.e., we perform max pooling over each of the convolved filters.

\cutsubsectionup
\subsection{Compressive Sensing} 
\label{sec:comp-sense}
\cutsubsectiondown

In compressive sensing, we assume that there is a latent sparse code $\vctz$ that generates the visible signal $\vctx$.
We say that a $p \times q$ matrix $\mtx{\Phi}$ with $q > p$ satisfies the Restricted Isometry Property RIP$(k,\delta_k)$ if there is a distortion factor $\delta_k > 0$ such that for all $\vctz \in \R^q$ with exactly $k$ non-zero entries, 
$(1 - \delta_k) \|\vctz\|^2_2 \leq \|\mtx{\Phi} \vctz\|^2_2 \leq (1 + \delta_k) \|\vctz\|^2_2$.
If $\mtx{\Phi}$ satisfies RIP with sufficiently small $\delta_k$ and if $\vctz$ is $k$-sparse, then given the vector $\vctx = \mtx{\Phi} \vctz \in \R^p$, we can efficiently recover $\vctz$ (see~\citet{candes2008restricted} for more details)\footnote{We note that this is a sufficient condition and that there are other, less restrictive sufficient conditions, as well as more complicated necessary conditions.
}.
There are many efficient algorithms, including $\ell_1$ sparse coding (e.g., $\ell_2$ minimization with $\ell_1$ regularization) and greedy and iterative algorithms, such as Iterative Hard Thresholding (IHT) \citep{Blumensath:2009tm}.

\noindent\textbf{Model-based compressive sensing.}
While sparse signals are a natural model for some applications, they are less realistic for CNNs.
We consider a vector $\vctz \in \R^{Kn}$ as the true sparse code generating the CNN input $\vctx$ with a particular model of sparsity. Rather than permitting $k$ non-zero entries anywhere in the vector $\vctz$, we divide the support of $\vctz$ into $K$ contiguous blocks of size $n$ and we stipulate that from each block there is at most one non-zero entry in $\vctz$ with a total of $k$ non-zero entries.
We call a vector with this sparsity model \emph{model-$k$-sparse} and denote the union of all $k$-sparse subspaces with this structure ${\cal M}_k$.
It is clear that ${\cal M}_k$ contains $n^k \binom{K}{k}$ subspaces. In our analysis, we consider linear combinations of two model-$k$-sparse signals.
To be precise, suppose that $\vctz = \alpha_1 \vctz_1 + \alpha_2 \vctz_2$ is the linear combination of two elements in ${\cal M}_k$.
Then, we say that $\vctz$ lies in the linear subspace ${\cal M}^2_k$ that consists of all linear combinations of vectors from ${\cal M}_k$.\footnote{Intuitively, ${\cal M}^2_k$ is a subspace where the error signal $\vcthatz - \vctz$ lies in and used for reconstruction bound derivation; see Appendix~\ref{sec:modelRIP}.}
We say that a matrix $\mtx{\Phi}$ satisfies the \emph{model-RIP} 
if there is a distortion factor $\delta_k > 0$ such that, for all $\vctz \in {\cal M}_k$, 
\begin{equation}
(1 - \delta_k) \|\vctz\|^2_2 \leq \|\mtx{\Phi} \vctz\|^2_2 \leq (1 + \delta_k) \|\vctz\|^2_2.\label{eq:model-rip}
\end{equation}
See~\citet{Baraniuk:2010hg} for the definitions of model sparse and model-RIP, as well as the necessary modifications to account for signal noise and compressible (as opposed to exactly sparse) signals, which we don't consider in this paper to keep our analysis simple.
Intuitively, a matrix satisfying the model-RIP is a nearly orthonormal matrix of a particular set of sparse vectors with a particular sparsity model or pattern.

For our analysis, we also need matrices $\mtx{\Phi}$ that satisfy the model-RIP for vectors $\vctz \in {\cal M}^2_k$.
We denote the distortion factor $\delta_{2k}$ for such matrices;
note that $\delta_k \leq \delta_{2k} < 1$.

\begin{algorithm}
	\caption{\small Model-based IHT}
	\label{alg:model-based-iht}
	\begin{small}
	\begin{algorithmic}[1]
		\REQUIRE model-RIP matrix $\mtx{\Phi}$, measurement $\vctx$ $(= \mtx{\Phi}\vctz)$, structured sparse approximation algorithm $\mathbb M$
		\ENSURE $k$-sparse approximation $\vcthatz$
		\STATE Initialize $\vcthatz_0 = 0$, $\vct{d} = \vctx$, $i = 0$ 
		\WHILE{stopping criteria not met}
			\STATE $i \leftarrow i+1$ 
			\STATE $\vct{b} \leftarrow \vcthatz_{i-1} + \mtx{\Phi}^T \vct{d}$ 
			\STATE $\vcthatz_i \leftarrow {\mathbb M}(\vct{b}, k)$ 
			\STATE $\vct{d} \leftarrow \vctx - \mtx{\Phi} \vcthatz_i$ 
		\ENDWHILE
	 	\STATE \textbf{return} $\vcthatz \leftarrow \vcthatz_i$
	\end{algorithmic}
	\end{small}
\end{algorithm}

Many efficient algorithms have been proposed for sparse coding and compressive sensing~\citep{olshausen1996emergence,mallat1993,beck2009}.
As with traditional compressive sensing, there are efficient algorithms for recovering model-$k$-sparse signals from measurements~\citep{Baraniuk:2010hg}, assuming the existence of an efficient structured sparse approximation algorithm $\mathbb M$, that given an input vector and the sparsity parameter, returns the vector closest to the input with the specified sparsity structure.

In CNNs, the max pooling operator finds the downsampled activations that are closest to the activations of the original size by retaining the most significant values.
The max pooling can be viewed as two steps: 1)~zeroing out the locally non-maximum values; 2)~downsampling the activations with the locally maximum values retained. 
To study the pooled activations with sparsity structures, we can recover dimension loss from the downsampling step by an unsampling operator. 
This procedure defines our structured sparse approximation algorithm $\vcthatz = \mathbb{M}(\vct{h},k)$, 
where $\vct{h}$ is the original (unpooled) response, 
$k$ is the sparsity parameter for block-sparsification, and
$\vcthatz$ is the sparsified response after pooling but without shrinking the length (i.e., the locally non-maximum values are zeroed out such that $\vcthatz$ has the same dimension as $\vct{h}$). 
Note that $\vcthatz$ is a model-$k$-sparse signal by construction. 
On the other hand, without considering the block-sparsification, we actually apply the following max pooling and upsampling operations:
\begin{equation}
\vcthatz = \operatorname{upsample}( \operatorname{max-pool}(\vct{h}), \vct{s} ),
\label{eq:sparse-approx}
\end{equation}
where $\vcthatz$ is the pooled response, $\vct{h}$ is the filter response of CNN given input before max pooling (see Section~\ref{sec:cnn}), and $\vct{s}$ denotes the upsampling switches that indicate whereto place the non-zero values in the upsampled activations.
Since our theoretical analysis does not depend on $\vct{s}$ but depends on $k$, any type of valid upsampling switches will be consistent with the block-sparsification (model-$k$-sparse) assumption, thus we will use $\mathbb{M}(\vct{h},k)$ to denote the structured sparse approximation algorithm \eqref{eq:sparse-approx} without worrying about $\vct{s}$.

We use model-sparse version of IHT \citep{Blumensath:2009tm} as our recovery algorithm, 
as one iteration of IHT for our model of sparsity captures exactly a feedforward CNN.\footnote{Multiple iterations of IHT can improve the quality of signal recovery. However, it is rather equivalent to the recurrent version of CNNs and does not fit to the scope of this work.}
Algorithm~\ref{alg:model-based-iht} describes the model-based IHT algorithm.
In particular, the sequence of steps 4--6 in the middle IHT
is exactly one layer of a feedforward CNN.
As a result, the theoretical analysis of IHT for model-based sparse signal recovery serves as a guide for how to analyze the approximation activations of a CNN. 

\commenthl{Cite sparse coding, MP and ISTA, etc. papers. Say they are also related to Alg 1. No harm in citing those papers.}
\textgray{when talking about convolutional networks, the paper does not cite its original paper in the late 80s. Furthermore, the paper omits entirely the research on (predictive) sparse coding, ignoring classic algorithms like marching-pursuit (MP) and iterative shrinkage-thresholding algorithm (ISTA) while posing algorithm 1 and algorithm 2 like something new without proper discussions. Also, the citations for compressive sensing are mostly on some survey papers which fails to describe the origin of algorithms and ideas.\\
Comments: (not addressed)\\
\textgray{<<Yi: I think this's been addressed in the previous discussion. (See violet above)>>}
\textgray{<<Kibok: Regarding the opinion that we pose alg 1 and 2 like something new, alg 1 should be fine since we cite Blumensath \& Davies (2009).
However, I think alg 2 needs citation or some statement that this is not novel thing.
Maybe we can cite any sparse coding paper, if it is relevant. >>}\\
The point is NOT to come up with new algorithms for reconstruction nor do we claim novelty of IHT (Algorithm 1). Indeed, the point of the paper is to show that the inner loop or one iteration of these classic sparse reconstruction algorithms are consistent with each layer of a CNN. We did not refer to other classic algorithms because IHT is the closest algorithm. (Although we note that it is an iterative hard thresholding rather than shrinkage algorithm.)\\
}

 \section{Analysis}
\label{sec:analysis}
\cutsectiondown

Following the idea of compressive sensing in Section~\ref{sec:comp-sense}, we assume that the input $\vctx$ is generated from a latent model-$k$-sparse signal $\vctz$ with basis vectors $\mtx{\Phi}$, which turns out to be $\mtx{W}^T$ by Theorem~\ref{thm:modelRIP} (i.e., $\vctx = \mtx{W}^T \vctz$).
Therefore, our analysis views the output of CNN (with pooling) is a reconstruction of $\vctz$ (i.e., $\vcthatz = \mathbb{M}(\mtx{W}\vctx,k)$), and $\mtx{W}^T$ can be used to reconstruct $\vctx$ from $\vcthatz$:
that is, $\vcthatx = \mtx{W}^T \vcthatz$.

\subsection{CNN Filters with Positive and Negative Pairs}
\label{sec:posneg}

Here we assume that all of the entries in the vectors are real numbers rather than only non-negative like when using $\relu$.
This setup is equivalent to using Concatenated $\relu$ (CReLU) \citep{crelu} as an activation function
(i.e., keeping the positive and negative activations as separate hidden units) with tied decoding weights.
The CReLU activation scheme is justified by the fact that trained CNN filters come in positive and negative pairs and that it achieves superior classification performance in several benchmarks.
This setting makes a CNN much easier to analyze within the model compressed sensing framework.

To motivate 
the setting,
we begin with a simple example.
Suppose that the matrix $\mtx{W}$ is an orthonormal basis for $\R^{MD}$ and define $\mtx{\Psi} = \begin{bmatrix} \mtx{W}^{T} & - \mtx{W}^{T} \end{bmatrix}$.
\begin{prop}
A one-layer CNN using the matrix $\mtx{\Psi}^{T}$, with no pooling, gives perfect reconstruction (with the matrix $\mtx{\Psi}$) for any input vector $\vctx \in \R^{MD}$. 
\end{prop}
\begin{proof}
Because we have both the positive and the negative dot products of the signal with the basis vectors in $\relu( \mtx{\Psi}^{T} \vctx ) = \relu\left( \begin{bmatrix} \mtx{W} \vctx \\ -\mtx{W} \vctx \end{bmatrix} \right)$, we have positive and negative versions of the hidden units
$\vcthid_{+} = \relu( \mtx{W} \vctx )$ and $\vcthid_{-} = \relu(- \mtx{W} \vctx )$
where we decompose $\vcthid = \mtx{W} \vctx = \vcthid_{+} - \vcthid_{-}$ into the difference of two non-negative vectors, the positive and the negative entries of $\vcthid$.
From this decomposition, we can easily reconstruct the original signal via
\begin{align*}
    \mtx{\Psi} \begin{bmatrix} \vcthid_{+} \\ \vcthid_{-} \end{bmatrix}
  	 &=  \begin{bmatrix} \mtx{W}^{T} & - \mtx{W}^{T} \end{bmatrix} \begin{bmatrix} \vcthid_{+} \\ \vcthid_{-} \end{bmatrix}
  	 =  \mtx{W}^{T} ( \vcthid_{+} - \vcthid_{-} ) \\
 	 &=  \mtx{W}^{T} \vcthid
 	 =  \mtx{W}^{T} \mtx{W} \vctx 
 	 = \vctx. 
\end{align*}
\vspace*{-0.1in}
\end{proof}

In the example above, we have pairs of vectors $(\vct{w}, -\vct{w})$ in our matrix $\mtx{\Psi}$.
Now suppose that we have a vector $\vctz$ where its positive and negative components can be split into $\vctz = \vctz_+ - \vctz_-$, and that we synthesize a signal $\vctx$ from $\vctz$ using the matrix $\begin{bmatrix} \mtx{W}^{T} & - \mtx{W}^{T} \end{bmatrix}$.
Then, we have
\[
	\begin{bmatrix}
		\mtx{W}^T & -\mtx{W}^T \\
	\end{bmatrix}
	\begin{bmatrix}
		\vctz_+ \\
		\vctz_- \\
	\end{bmatrix}
	= \mtx{W}^T (\vctz_+ - \vctz_- ) = \mtx{W}^T \vctz = \vctx.
\]
Next, we multiply $\vctx = \mtx{W}^T \vctz$ by a concatenation of positive and negative $\mtx{W}$, then we get
$
	\begin{bmatrix}
		\mtx{W} \\ -\mtx{W} \\
	\end{bmatrix}
		\vctx
	=
	\begin{bmatrix}
		\mtx{W} \mtx{W}^T \vctz \\
		- \mtx{W} \mtx{W}^T \vctz
	\end{bmatrix}
$
and if we apply $\relu$ to this vector, we get
$ 
	\begin{bmatrix}
		(\mtx{W} \mtx{W}^T \vctz)_+ \\
		(\mtx{W} \mtx{W}^T \vctz)_-
	\end{bmatrix}
$,
which is a vector $\mtx{W} \mtx{W}^T \vctz$ that is split into its positive and negative components.
The structure of the product $\mtx{W} \mtx{W}^T$ is crucial to the reconstruction quality of the vector $\vctz$. In addition, this calculation shows that if we have both positive and negative pairs of filters or vectors, then the $\relu$ function applied to both the positive and negative dot products simply splits the vector into the positive and negative components.
These components are then reassembled in the next computation.
For this reason, in the analysis in the following sections, it is sufficient to 
assume
that all of the entries in the vectors are real numbers, rather than only non-negative.
 
\cutsubsectionup
\subsection{Model-RIP and Random Filters} \label{sec:model-rip}
\cutsubsectiondown

Our first main result shows that if we use Gaussian random filters in our CNN, then, with high probability, $\mtx{W}^T$, the transpose of a large matrix formed by the convolution filters satisfies the model-RIP.
In other words, Gaussian random filters generate a matrix whose transpose $\mtx{W}^T$ is almost an orthonormal transform for sparse signals with a particular sparsity pattern (that is consistent with our pooling procedure).
The bounds in the theorem tell us that we must balance the size of the filters $\ell$ and the number of channels $M$ against the sparsity of the hidden units $k$, the number of the filter banks $K$, the number of shifts $n$, the distortion parameter $\delta_k$, and the failure probability $\epsilon$.
The proof is in Appendix~\ref{sec:modelRIP}.

\begin{theorem}
\label{thm:modelRIP}
	Assume that we have $MK$ vectors $\vct{w}_{i,m}$ of length $\ell$ in which each entry is a scaled i.i.d. (sub-)Gaussian random variable with zero mean and unit variance (the scaling factor is $1/\sqrt{M\ell})$.
	Let $t$ be the stride length (where $n = (D - \ell)/t + 1$) and $\mtx{W}$ be a structured random matrix, which is the weight matrix of a single layer CNN with $M$ channels and input length $D$. If 
	\[
		\frac{M \ell^2}{D} \geq \frac{C}{\delta_k^2} \Big( k (\log(K) + \log(n)) - \log(\epsilon)\Big)
	\]
	for a positive constant $C$, then with probability $1 - \epsilon$, the $MD \times Kn$ matrix $\mtx{W}^T$ satisfies the model-RIP for model ${\cal M}_k$ with parameter $\delta_k$.
\end{theorem}

We also note that the same analysis can be applied to the sum of two model-$k$-sparse signals, with changes in the constants (that we do not track here).
\begin{corollary}
Random matrices with the CNN structure satisfy, with high probability, the model-RIP for ${\cal M}^2_k$.
\end{corollary}

Other examples of matrices that satisfy the model-RIP 
include wavelets and localized Fourier bases;
both examples can be easily and efficiently implemented via convolutions.

\cutsubsectionup
\subsection{Reconstruction Bounds}
\cutsubsectiondown

Suppose $\mtx{W}^T$ satisfies the model-RIP and $\vcthatz$ is the reconstruction of true sparse code $\vctz$ through a CNN layer followed by pooling, i.e., $\vcthatz=\mathbb{M}(\mtx{W}\vctx,k)$.
Then, Theorem~\ref{thm:reconstruction} shows that 
$\vcthatx = \mtx{W}^T \vcthatz$ is an approximate reconstruction of the input signal, and the relative error is bounded on a function of the distortion parameters of the model-RIP.
\begin{theorem} 
\label{thm:reconstruction}
	We assume that $\mtx{W}^T$ satisfies the ${\cal M}_k^2$-RIP with constant $\delta_k \leq \delta_{2k} < 1$.
	If we use $\mtx{W}$ in a single layer CNN both to compute the hidden units $\vcthatz$ and to reconstruct the input $\vctx$ from these hidden units as $\vcthatx$ so that $\vcthatx=\mtx{W}^T\mathbb{M}(\mtx{W}\vctx,k)$, the error in our reconstruction is  
	\[
		\|\vcthatx - \vctx\|_2 \leq \frac{5 \delta_{2k}}{1 - \delta_{k}}
					\frac{\sqrt{1 + \delta_{2k}}}{\sqrt{1 - \delta_{2k}}}  \|\vctx\|_2.
	\]
\end{theorem}
See Appendix~\ref{sec:reconstruct} for the detailed proofs.
Part of our analysis also shows that the hidden units $\vcthatz$ are approximately the putative coefficient vector $\vctz$ in the sparse linear representation for the input signal.
Recall that the structured sparsity approximation algorithm $\mathbb{M}$ includes the downsampling caused by pooling and an unsampling operator. 
Theorem~\ref{thm:reconstruction} is applicable to any type of upsampling switches, so our reconstruction bound is generic to the particular design choice on how to recover the activation size in a decoding neural network.
We can extend the analysis for a single layer CNN to multiple layer CNN by using the output on one layer as the input to another, following the proof in Appendix~\ref{sec:reconstruct}. 
We leave further investigation of this idea as future work.

 \cutsectionup
\section{Experimental Evidence and Analysis} \label{sec:experiments}
\cutsectiondown

In this section, we provide experimental validation of our theoretical model and analysis. 
We first validate the practical relevance of our assumption by examining the effectiveness of random filter CNNs,
and then provide results on more realistic scenarios.
In particular, we study popular deep CNNs trained for image classification on ILSVRC 2012 dataset~\citep{imagenet}.
We calculate empirical model-RIP bounds for $\mtx{W}^T$, showing that they are consistent with our theory. 
Our results are also consistent with a long line of research shows that it is reasonable to model real and natural images as sparse linear combinations overcomplete dictionaries \citep{boureau2008sparse,le2013building,lee2008sparse,olshausen1996emergence,ranzato2007unsupervised,yang2010image}. 
In addition, we verify our theoretical bounds for the reconstruction error $\| \vctx - \mtx{W}^T \vcthatz \|_2/\| \vctx\|_2$ on real images.
We investigate both randomly sampled filters and empirically learned filters in these experiments.  
Our implementation is based on Caffe~\citep{caffe} and MatConvNet~\citep{matconvnet}. 

Recall that our theoretical analysis is generic to any upsampling switches in \eqref{eq:sparse-approx} for reconstruction. 
In the experiments, we specifically use the naive upsampling to reverse max-pool activations to its original size, where only the first element in a pooling region is assigned with the pooled activation, and the rest elements are all zero. 
Thus, no extra information other than the pooled activation values are taken into account.

\textgray{The reconstruction error is the relative $l_2$ distance between the original image and the reconstruction. It is the metric we use in our theoretical and empirical analysis. A smaller value is better. Fig.4 and Fig.5(supp) show the visual quality. -> well described in the paragraph above\\
A number of works have shown that natural images are well approximated by sparse coding model. We benchmark initially one part of our model with random sparse signals and further test the model with real images. -> well described in the paragraph above
}

\cutsubsectionup
\subsection{Gaussian Random CNNs on CIFAR-10}
\label{sec:exp-randomfilters}
\cutsubsectiondown
To show the practical relevance of our theoretical assumptions on using random filters for CNNs as stated in Section~\ref{sec:relatedwork-randomfilters}, we evaluate simple CNNs with Gaussian random filters with i.i.d. zero mean unit variance entries on the CIFAR-10~\citep{cifar10}.
Note that the goal of this experiment is not to achieve state-of-the-art results, but to examine practical relevance of our assumption on random filter CNNs.
Once the CNNs weights are initialized (randomly), they are fixed during the training of the classifiers.\footnote{Implementation detail: we add a batch normalization layer together with a learnable scale and bias before the activation so that we do not need to tune the scale of the filters. See Appendix~\ref{sec:cifar-random-filters-detailed} for more details.}
Specifically, we test random CNNs with 1, 2, and 3 convolutional layers followed by ReLU activation and $2 \times 2$ max pooling layer.
We tested different filter sizes ($3,5,7$) and numbers of channels ($64,128,256,1024,2048$) and report the best classification accuracy 
by cross-validation in Table~\ref{tab:cifar-random-filters}.
We also report the best performance using learnable filters for comparison.
More details about the architectures can be found in Appendix~\ref{sec:cifar-random-filters-detailed}.
We observe that CNNs with Gaussian random filters achieve 
good classification performance (implying that they serve as reasonable representation of input data), 
which is not too far off the learned filters.
Our experimental results are also consistent with the observations made by \citet{jarrett2009best} and \citet{saxe2011random}.
In conclusion, those results suggest that CNNs with Gaussian random filters might be a reasonable setup which is amenable to mathematical analysis while not being too far off in terms of practical relevance.

\begin{table}
\begin{centering}
\begin{footnotesize}
\begin{tabular}{c||c|c|c}
\hline 
Method & 1 layer & 2 layers  & 3 layers \tabularnewline
\hline 
Random filters & 66.5\%  & 74.6\%  & 74.8\%\tabularnewline
\hline 
Learned filters & 68.1\%  & 83.3\%  & 89.3\% \tabularnewline
\hline 
\end{tabular}
\end{footnotesize}
\par\end{centering}
\caption{Classification accuracy of CNNs with random and learnable filters on CIFAR-10.  A typical layer consists of four operators: convolution, ReLU, batch normalization and max pooling. Networks with optimal filter size and numbers of output channels are used. (See Appendix~\ref{sec:cifar-random-filters-detailed} for more details about the architectures). The random filters, assumed in our theoretical analysis, perform reasonably well, not far off the learned filters. }
\label{tab:cifar-random-filters}
\end{table}

\begin{figure}
\begin{centering}
\hfill{}\subfloat[$\| \mtx{W}^T \vctz\|_2 /  \|\vctz\|_2$]{\includegraphics[width=0.33\columnwidth]{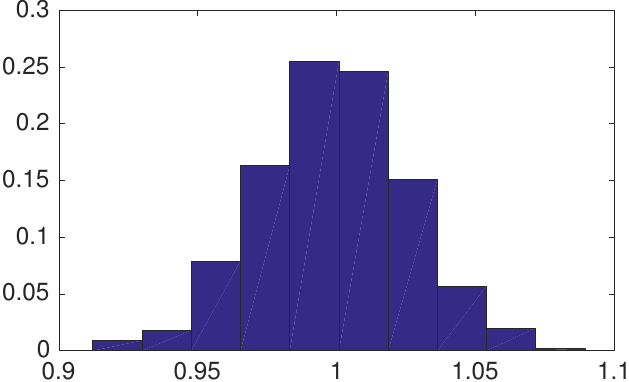}}
\hfill{}\subfloat[{\scriptsize $\| \mtx{W} \mtx{W}^T \vctz\|_2 /  \|\vctz\|_2$}]{\includegraphics[width=0.33\columnwidth]{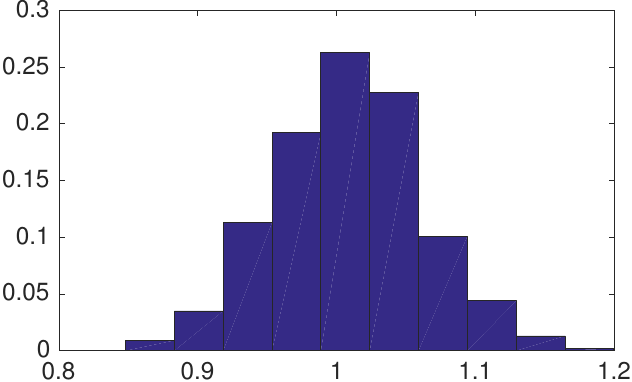}}
\hfill{}\subfloat[$\|\vcthatx - \vctx \|_2 /  \|\vctx\|_2$]{\includegraphics[width=0.33\columnwidth]{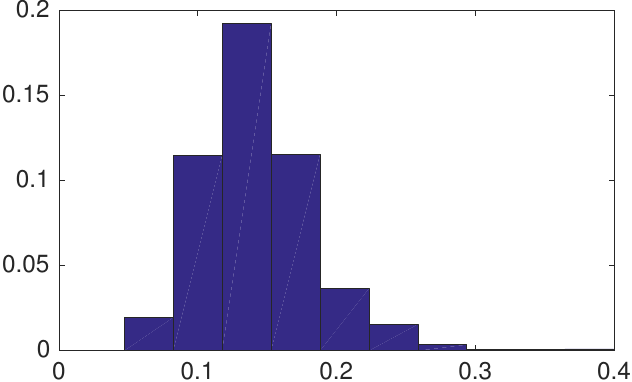}}
\hfill{}
\par\end{centering}
\caption{For 1-d scaled Gaussian random filters $\mtx{W}$, we plot the histogram of ratios (a) $\| \mtx{W}^T \vctz\|_2 /  \|\vctz\|_2$ (model-RIP in \eqref{eq:model-rip}; supposed to be concentrated at $1$), (b) $\| \mtx{W} \mtx{W}^T \vctz\|_2 /  \|\vctz\|_2$ 
(ratio between the norm of the reconstructed code $\mtx{W} \mtx{W}^T \vctz$ and that of the original code $\vctz$; supposed to be concentrated at $1$), and
(c) $\|\vcthatx - \vctx \|_2 /  \|\vctx\|_2$ (reconstruction bound in Theorem~\ref{thm:reconstruction}, supposed to be small),  where $\vctz$ is a ${\cal M}_k$ sparse signal that generates the vector $\vctx$ and $\vcthatx=\mtx{W}^T\mathbb{M}(\mtx{W}\vctx,k)$ is the reconstruction of $\vctx$, where we use the naive unsampling to recover the reduced dimension due to pooling:
we place recovered values in the top-left corner in each unsampled block.
(See Section~\ref{sec:comp-sense}).}
\label{fig:1D_sanity_check}
\end{figure}

\cutsubsectionup
\subsection{1-d Model-RIP} \label{sec:1d-exp}
\cutsubsectiondown

We use 1-d synthetic data to empirically show the basic validity of our theory in terms of the model-RIP in \eqref{eq:model-rip} and reconstruction bound in Theorem~\ref{thm:reconstruction}. 
We plot the histograms of the empirical model-RIP values of 1-d Gaussian random filters $\mtx{W}$ ( scaled by $1/\sqrt{lM}$ ) with size $l\times1\times M\times K = 5\times1\times32\times96$ on 1-d ${\cal M}_k$ sparse signal $\vctz$ with size $D=32$ and sparsity $k=10$, whose non-zero elements are drawn from a uniform distribution on $[-1,1]$.   
The histograms in Figure~\ref{fig:1D_sanity_check}~(a)--(b) are tightly centered around $1$, suggesting that $\mtx{W}^T$ satisfies the model-RIP in \eqref{eq:model-rip} and its corollary from Lemma~\ref{lma:rip-bound}, respectively.
We also empirically show the reconstruction bound in Theorem~\ref{thm:reconstruction} on synthetic vectors $\vctx=\mtx{W}^T\vctz$ (Figure~\ref{fig:1D_sanity_check}~(c)).
The reconstruction error is concentrated at around $0.1$--$0.2$ and bound under $0.5$. 
Results in Figure~\ref{fig:1D_sanity_check} suggest the practical validity of our theory when the model assumptions hold.

\cutsubsectionup
\subsection{Architectures for 2-d Model-RIP}
\cutsubsectiondown

We conduct the rest of our experimental evaluations on the 16-layer VGGNet (Model D in \citet{vggnet}),
where the computation is carried out on images;
e.g., convolution with a 2-d filter bank and pooling on square regions. 
In contrast to the theory, the realistic network does not pool activations over all the possible shifts for each filter, but rather on non-overlapping patches.
The networks are trained for the large-scale image classification task, which is important for extending to other supervised tasks in vision. 
The main findings on VGGNet are presented in the rest of this section;
we also provide some analysis on AlexNet~\citep{alexnet} in Appendix~\ref{sec:alexnet_sparsity}.

VGGNet contains five macro layers of convolution and pooling layers, and each macro layer has 2 or 3 convolutional layers followed by a pooling layer. 
We denote the $j$-th convolutional layer in the $i$-th macro layer ``conv$(i,j)$,'' and the pooling layer ``pool$(i)$.''
The activations/features from $i$-th macro layer are the output of pool$(i)$.
Our analysis is for single convolutional layers.

\cutsubsectionup
\subsection{2-d Model-RIP}
\label{sec:exp-2d-rip}
\cutsubsectiondown

The key to our reconstruction bound is Theorem~\ref{thm:reconstruction}.
We empirically evaluate the model-RIP, i.e., $\Vert \mtx{W}^T\vctz \Vert / \Vert \vctz \Vert$, for real CNN filters of the pretrained VGGNet. 
We use two-dimensional coefficients 
$\vctz$ (each block of coefficients is of size $D \times D$), $K$ filters of size $\ell \times \ell$, and pool the coefficients over smaller pooling regions (i.e., not over all possible shifts of each filter). 
The following experimental evidence suggests that the sparsity model and the model-RIP of the filters are consistent with our mathematical analysis on the simpler one-dimensional case.

To check the significance of the model-RIP (i.e., how close $\Vert \mtx{W}^T\vctz \Vert / \Vert \vctz \Vert  $ is to $1$) in controlled settings, we first synthesize the hidden activations $\vctz$ with sparse uniform random variables, which fully agree with our model assumptions.

\begin{table}
	\begin{small}
		\begin{center}
						\begin{tabular}{>{\colcut}c<{\colcut}||>{\colcut}c<{\colcut}|>{\colcut}c<{\colcut}|>{\colcut}c<{\colcut}|>{\colcut}c<{\colcut}|>{\colcut}c<{\colcut}|>{\colcut}c<{\colcut}}
				\hline
				layer & c(1,1) & c(1,2) & p(1) & c(2,1) & c(2,2) & p(2) \\
				\hline
				\% of non-zeros & 49.1 & 69.7 & 80.8 & 67.4 & 49.7 & 70.7 \\
				\hline
				\hline
				layer & c(3,1) & c(3,2) & c(3,3) & p(3) & c(4,1) & c(4,2) \\
				\hline
				\% of non-zeros & 53.4 & 51.9 & 28.7 & 45.9 & 35.6 & 29.6 \\
				\hline
				\hline
				layer & c(4,3) & p(4) & c(5,1) & c(5,2) & c(5,3) & p(5) \\
				\hline
				\% of non-zeros & 12.6 & 23.1 & 23.9 & 20.6 & 7.3 & 13.1 \\
				\hline
			\end{tabular}
		\end{center}
	\end{small}
	\caption{Layer-wise sparsity of VGGNet on ILSVRC 2012 validation set.
		``c'' stands for convolutional layers and ``p'' represents pooling layers.
		CNN with random filters in Section~\ref{sec:exp-2d-rip} can be simulated with the same sparsity. }
	\label{tab:vgg_sparsity}
\end{table}

\begin{table}[t]\begin{small}
\begin{center}
\begin{tabular}{>{\colcut}c<{\colcut}||>{\colcut}c<{\colcut}|>{\colcut}c<{\colcut}|>{\colcut}c<{\colcut}|>{\colcut}c<{\colcut}|>{\colcut}c<{\colcut}|>{\colcut}c<{\colcut}|>{\colcut}c<{\colcut}}
\hline
layer & (1,1) & (1,2) & (2,1) & (2,2) & (3,1) & (3,2) & (3,3)  \\
\hline
learned & 0.943 & 0.734 & 0.644 & 0.747 & 0.584 & 0.484 & 0.519 \\
\hline
random & 0.670 & 0.122 & 0.155 & 0.105 & 0.110 & 0.090 & 0.080 \\
\hline
\hline
layer & (4,1) & (4,2) & (4,3) & (5,1) & (5,2) & \multicolumn{1}{c}{(5,3)} \\
\cline{1-7}
learned & 0.460 & 0.457 & 0.404 & 0.410 & 0.410 & \multicolumn{1}{c}{0.405} \\
\cline{1-7}
random & 0.092 & 0.062 & 0.062 & 0.070 & 0.067 & \multicolumn{1}{c}{0.067} \\
\cline{1-7}
\end{tabular}
\end{center}
\caption{Comparison of coherence between learned filters in each convolutional layer of VGGNet and Gaussian random filters with corresponding sizes.}
\label{tab:vgg_coherence}
\end{small}
\end{table}

The sparsity of $\vctz$ is constrained to the average level of the real CNN activations, which is reported in Table~\ref{tab:vgg_sparsity}.  
Given the filters of a certain convolutional layer, we use the synthetic $\vctz$ (in equal position to this layer's output activations) to get statistics for the model-RIP. 
To be consistent with succeeding experiments, we choose conv$(5,2)$, while other layers show similar results. 
Figure~\ref{fig:RIP_hidden_feature}~(a) summarizes the distribution of empirical model-RIP values, which is clearly centered around $1$ and satisfies \eqref{eq:model-rip} with a short tail roughly bounded by $\delta_k<1$. 
For more details of the algorithm, we normalize the filters from the conv$(5,2)$ layer, which are $\ell \times \ell$ ($\ell = 3$).
All $K=512$ filters with $M=512$ input channels are used.\footnote{We do not remove any filters including those in approximate positive/negative pairs (see Section~\ref{sec:analysis}.)}
We set $D = 15$, which is the same as the output activations of conv$(5,2)$, and use $2 \times 2$ pooling regions\footnote{No pooling layer follows conv$(5,2)$ in VGGNet.
However, we use it in this way to analyze the convolution-pooling pair per theory.}, which is commonly used in recent CNNs. 
We generate 1000 ${\cal M}_k$ randomly sampled sparse activation maps $\vctz$ by first sampling their non-zero supports and then filling elements on the supports uniformly from $[-1,1]$. The sparsity is the same as that in conv$(5,1)$ activations.
\commenthl{Are we using 262144 filters? Doesn't make sense to me.} 

\begin{figure}[t]
\begin{small}
\vspace*{-0.1in}
\begin{centering}
\hfill{}\subfloat[Random]{\includegraphics[width=0.33\columnwidth]{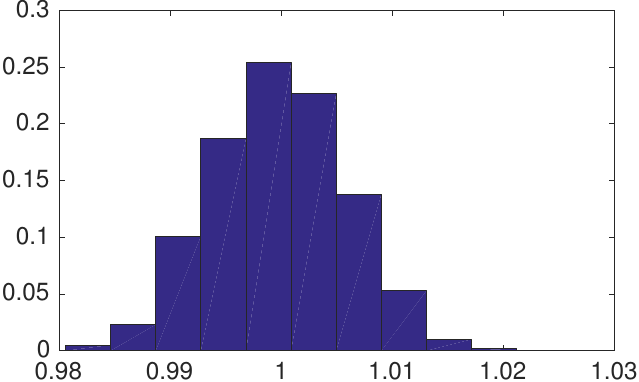}}
\hfill{}\subfloat[After $\relu$]{\includegraphics[width=0.33\columnwidth]{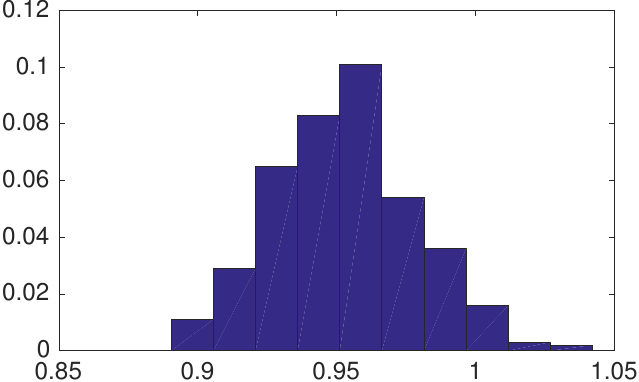}}
\hfill{}\subfloat[Before $\relu$]{\includegraphics[width=0.33\columnwidth]{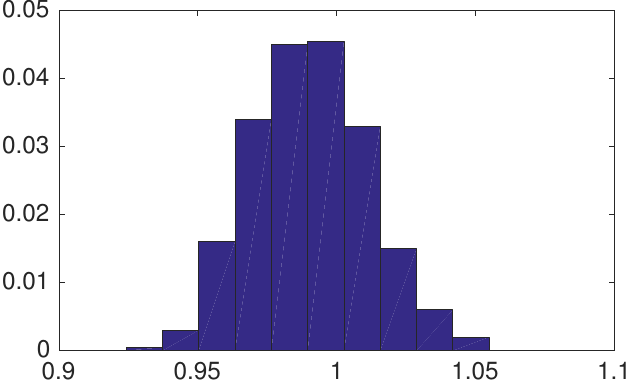}}
\hfill{}
\par\end{centering}
\caption{For VGGNet's conv$(5,2)$ filters $\mtx{W}$, we plot the histogram of ratios $\| \mtx{W}^T \vctz\|_2 /  \|\vctz\|_2$, which is expected to be concentrated at $1$ according to \eqref{eq:model-rip}, where $\vctz$ is a ${\cal M}_k$ sparse signal.
In (a), $\vctz$ is randomly generated with the same sparsity as the conv$(5,2)$ activations and from a uniform distribution for the non-zero magnitude.
In (b) and (c),
$\vctz$ is recovered by Algorithm~\ref{alg:rip-act} from the conv(5,1) activations before and after applying $\relu$, respectively.
The learned filters admits similar model-RIP value distributions to the random filters except for a bit larger bandwidth, which means the model-RIP in \eqref{eq:model-rip} can empirically hold even when the filters do not necessarily subject to the i.i.d Gaussian random assumption. }
\label{fig:RIP_hidden_feature}
\end{small}
\end{figure}

\begin{algorithm}[t]
	\caption{\small Sparse hidden activation recovery}
	\label{alg:rip-act}
	\begin{small}
	\begin{algorithmic}[1]
		\REQUIRE convolution matrix $\mtx{W}$, input activation/image $\vctx$
		\ENSURE hidden code $\vctz$, satisfying our model-RIP assumption with $\mathcal{M}_k$ and reconstructing $\vctx$ with $\mtx{W}$
		\STATE $\vctz^{\mathrm{init}} = \argmin_{\vctz}{ {\Vert \vctx-\mtx{W}^T\vctz\Vert}^2_2 + \lambda { \Vert \vctz \Vert}_1 }$
		\STATE $\vctz^{\mathrm{model}} = \operatorname{up-sample}( \operatorname{max-pool}(\vctz^{\mathrm{init}}),\vct{s})$, \\
		where $\vct{s}=\operatorname{pool-switch}(\vctz)$ 
		\STATE $\vctz = \argmin_{\vctz}{ {\Vert \vctx-\mtx{W}^T\vctz\Vert}^2_2 + \lambda { \Vert \vctz \Vert}_1 }, $ \\
						$\mbox{s.t. } \vctz_i = 0 \mbox{ if } \vctz^{\mathrm{model}}_i = 0$
	\end{algorithmic}
	\end{small}
\end{algorithm}

More realistically, we observe that the actual conv$(5,2)$ activations from VGGNet are not necessarily drawn from a model-sparse uniform distribution. 
This motivates us to evaluate the empirical model-RIP on the hidden activations $\vctz$ that reconstruct the actual input activations $\vctx$ from conv$(5,1)$ by $\mtx{W}^T\vctz$. 
Per theory, the $\vctx$ is given by a max pooling layer, so we constrain the sparsity (i.e., the size of the support set is no more than $1$ in a pooling region for a single channel). 
We use a simple and efficient algorithm to recover $\vctz$ from $\vctx$ in Algorithm~\ref{alg:rip-act}.
The algorithm is inspired by ``$\ell_1$ heuristic" method that are commonly used in practice (e.g., ~\citet{boyd2015}).
As shown in Algorithm~\ref{alg:rip-act}, we first do $\ell_1$-regularized least squares without constraining the support set. 
Max pooling is then applied to figure out the support set for each pooling region. 
In particular, we use max pooling and unpooling with known switches (line~2)
to zero out the locally non-maximum values without messing up the support structures. 
We perform $\ell_1$-regularized least squares again on the fixed support set to recover the hidden activations satisfying the model sparsity. 
As shown in Figures~\ref{fig:RIP_hidden_feature}~(b)--(c), the empirical model-RIP values for visual activations $\vctx$ from conv$(5,1)$ with/without $\relu$ are both close to $1$. 
The center offset to $1$ is less than $0.05$ and the range bound $\delta_k$ is roughly less than $0.05$, which agrees with the theoretical bound in (\ref{eq:model-rip}). 
To gain more insight, we summarize the learned filter coherence in Table~\ref{tab:vgg_coherence} for all convolutional layers in VGGNet.\footnote{The coherence is defined as the maximum (in absolute value) dot product between distinct pairs of columns of the matrix $W^T$, i.e. $\mu=\max_{i\not=j}|W_i W_j^T|$, where $W_i$ denote the $i$-th row of matrix $W$.}
This measures the correlation or similarity between the columns of $W^T$ and is a proxy for the value of the model-RIP parameter $\delta_k$ (which we can only estimate computationally).
The smaller the coherence, the smaller $\delta_k$ is, and the better the reconstruction.
The coherence of the learned filters is not low, which is inconsistent with our theoretical assumptions.  
However, the model-RIP turns out to be robust to this mismatch. It demonstrates the strong practical invertibility of CNN.

\textgray{
Definition of coherence
We did not include the definition of coherence. Coherence is the maximum (in absolute value) dot product between distinct pairs of columns of the matrix $W^T$. This measures the correlation or similarity between the columns of $W^T$ and is a proxy for the value of the model-RIP parameter $\delta_k$ (which we can only estimate computationally). The smaller the coherence, the smaller $\delta_k$ is, and the better the reconstruction.
}

\cutsubsectionup
\subsection{Reconstruction Bounds}
\cutsubsectiondown

With model-RIP as a sufficient condition, Theorem~\ref{thm:reconstruction} provides a theoretical bound for layer-wise reconstruction via $\vcthatx=\mtx{W}^T\mathbb{M}(\mtx{W}\vctx)$, which consists of the projection and reconstruction in one IHT iteration. 
Without confusion, we refer to it as IHT for notational convenience. 
We investigate the practical reconstruction errors on 
pool$(1)$ to $(4)$ of VGGNet. 

\begin{figure}[t]
	\centering
	\includegraphics[width=1.0\columnwidth]{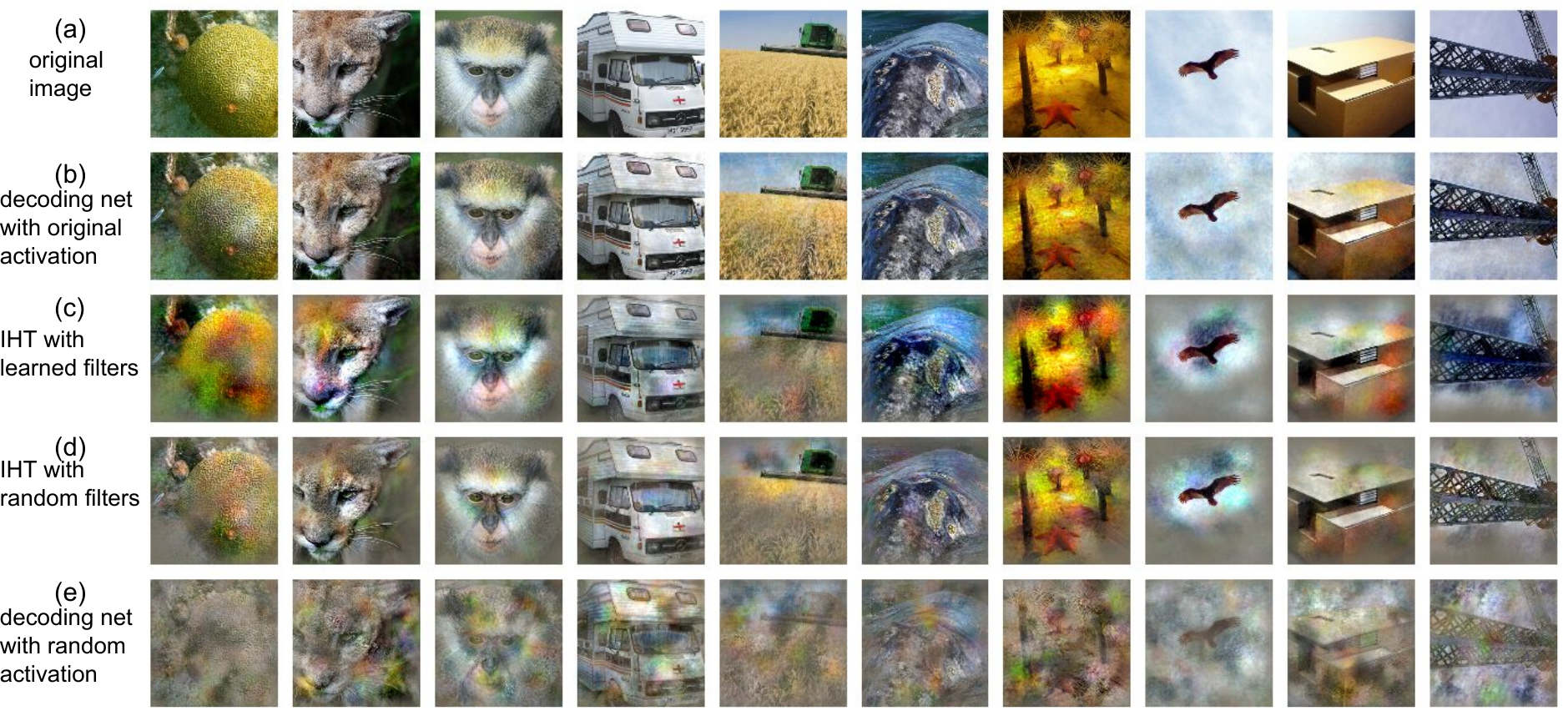}
	\caption{Visualization of images reconstructed by a pretrained decoding network with VGGNet's pool$(4)$ activation reconstructed using different methods: (a)~original image, (b)~output of the $5$-layer decoding network with original activation, (c) output of the decoding net with reconstructed activation by IHT with learned filters, (d) output of the decoding net with reconstructed activation by IHT with Gaussian random filters, and (e) output of the decoding net with Gaussian random activation.}
	\label{fig:recon_images}
\end{figure}

To encode and reconstruct intermediate activations of CNNs, we employ IHT with sparsity estimated from the real CNN activations on ILSVRC 2012 validation set (see Table~\ref{tab:vgg_sparsity}).  
We also reconstruct input images, since CNN inversion is not limited to a single layer, and images are easier to visualize than hidden activations.   
To implement image reconstruction, we project the reconstructed activations into the image space via a pretrained decoding network as in~\citet{deconv-recon}, which extends a similar autoencoder architecture as in \citet{invert-cnn} to a stacked ``what-where'' autoencoder~\citep{what-where}. 
The reconstructed activations were scaled to have the same norm as the original activations so that we can feed them into the decoding network. 

As an example, Figure~\ref{fig:recon_images} illustrates the image reconstruction results for the hidden activations of pool$(4)$.
Interestingly, the decoding network itself is quite powerful, since it can reconstruct the rough (although very noisy) glimpse of images 
with Gaussian random input, as shown in Figure~\ref{fig:recon_images}~(e). 
Object shapes are recovered up to some extent by using the pooling switches only in the ``what-where'' autoencoder. 
This result suggests that it is important to determine which pooling units are active and then to estimate these values accurately.
These steps are consistent with the steps in the inner loop of any iterative sparse signal reconstruction algorithm.

In Figure~\ref{fig:recon_images}~(c), we take the pretrained conv$(5,1)$ filters for IHT.
The images recovered from the IHT reconstructed pool$(4)$ activations are reasonable and the reconstruction quality is significantly better than the random input baseline. 
We also try Gaussian random filters (Figure~\ref{fig:recon_images}~(d)), which agree more with the model assumptions (e.g., lower coherence, see Table~\ref{tab:vgg_coherence}). The learned filters from VGGNet perform equally well visually. 
IHT ties the encoder and decoder weights (no filter learning for the decoder), so it does not perform as well as the decoding network trained with a huge batch of data (Figure~\ref{fig:recon_images}~(b)). 
Nevertheless, we show both theoretically and experimentally decent reconstruction bounds for these simple reconstruction methods on real CNNs. 
More visualization results for more layers are in Appendix~\ref{sec:recon-vgg}. 

\textgray{
The model along with our experimental analysis suggests that it is important to determine which pooling units are active and then to estimate these values accurately. These steps are consistent with the steps in the inner loop of any iterative sparse signal reconstruction algorithm.\\
Our reconstruction is mainly for analyzing the practical significance of the theory rather than beating [13]. -> is this really necessary to state in the main paper?
}

\begin{table}[t]\begin{small}
\begin{center}
\begin{tabular}{>{\colcutb}c<{\colcutb}||>{\colcutb}c<{\colcutb}|>{\colcutb}c<{\colcutb}|>{\colcutb}c<{\colcutb}|>{\colcutb}c<{\colcutb}|>{\colcutb}c<{\colcutb}|>{\colcutb}c<{\colcutb}}
\hline
\multirow{4}{*}{layer} & \multicolumn{3}{c|}{image space} & \multicolumn{3}{c}{activation space}\\
& \multicolumn{3}{c|}{relative error} & \multicolumn{3}{c}{relative error}\\
\hhline{~------}
& learned & random & random & learned & random & random \\
& filters & filters & activations & filters & filters & activations \\
\hline
1 & 0.423 & 0.380 & 0.610 & 0.895 & 0.872 & 1.414 \\
\hline
2 & 0.692 & 0.438 & 0.864 & 0.961 & 0.926 & 1.414 \\
\hline
3 & 0.326 & 0.345 & 0.652 & 0.912 & 0.862 & 1.414 \\
\hline
4 & 0.379 & 0.357 & 0.436 & 1.051 & 0.992 & 1.414 \\
\hline

\end{tabular}
\end{center}
\caption[]{Layer-wise relative reconstruction errors by different methods in activation space and image space between reconstructed and original activations. For macro layer $i$, we take its activation after pooling from that layer and reconstruct it with different methods (using learned filters from the layer above or scaled Gaussian random filters) and feed the reconstructed activation to a pretrained corresponding decoding network.\footnotemark}
\label{tab:recon_error}
\end{small}
\end{table}
\footnotetext{The values in the last column are identical ($=1.414$) for all layers because $\Vert\vct{f} - \hat{\vct{f}}\Vert/\Vert\vct{f}\Vert = \sqrt{2}$ on average for Gaussian random $\hat{\vct{f}}$ provided $\Vert\vct{f}\Vert = \Vert\hat{\vct{f}}\Vert$.}

In Table~\ref{tab:recon_error}, we summarize reconstruction performance for all 4 macro layers.
With random filters, the model assumptions hold and the IHT reconstruction is the best quantitatively.
IHT with real CNN filters performs comparable to the best case and much better than the baseline established by the randomly sampled activations.

 \cutsectionup
\section{Conclusion}
\cutsectiondown

We introduce three concepts that tie together a particular model of compressive sensing (and the associated recovery algorithms), the properties of learned filters, and the empirical observation that CNNs are (approximately) invertible.
Our experiments show that filters in trained CNNs are consistent with the mathematical properties we present while the hidden units exhibit a much richer structure than mathematical analysis suggests.
Perhaps simply moving towards a compressive, rather than exactly sparse, model for the hidden units will capture the sophisticated structure in these layers of a CNN or, perhaps, we need a more sophisticated model.
Our experiments also demonstrate that there is considerable information captured in the switch units (or the identities of the non-zeros in the hidden units after pooling) that no mathematical model has yet expressed or explored thoroughly. 
We leave such explorations as future work.
 
\section*{Acknowledgments}
This work was supported in part by ONR N00014-16-1-2928, NSF CAREER IIS-1453651, and Sloan Research Fellowship. We would like to thank Michael Wakin for helpful discussions about concentration of measure for structured random matrices.

\begingroup
\begin{small}
\setstretch{1}
\bibliographystyle{named}
\bibliography{ref,ref-from-yuting} 
\end{small}
\endgroup

\clearpage
\onecolumn
\appendix
\begin{center}
\LARGE Appendix: \\ { \renewcommand{\\}{} \mytitle }
\par\end{center}

\section{Mathematical Analysis: Model-RIP and Random Filters}
\label{sec:modelRIP}

\begin{reftheorem}{thm:modelRIP}{(Restated)}
	Assume that we have $MK$ vectors $\vct{w}_{i,m}$ of length $\ell$ in which each entry is a scaled i.i.d. (sub-)Gaussian random variable with zero mean and unit variance (the scaling factor is $1/\sqrt{M\ell})$.
	Let $t$ be the stride length (where $n = (D - \ell)/t + 1$) and $\mtx{W}$ be a structured random matrix, which is the weight matrix of a single layer CNN with $M$ channels and input length $D$.
	If 
	\[
		\frac{M \ell^2}{D} \geq \frac{C}{\delta_k^2} \Big( k (\log(K) + \log(n)) - \log(\epsilon)\Big)
	\]
	for a positive constant $C$, then with probability $1 - \epsilon$, the $MD \times Kn$ matrix $\mtx{W}^T$ satisfies the model-RIP for model ${\cal M}_k$ with parameter $\delta_k$.
\end{reftheorem}
\begin{proof}
We note that the proof follows the same structure of those in other papers such as \citet{Park:2011iw} and \citet{Vershynin:2010vk},
though we make minor tweaks to account for the particular structure of $\mtx{W}^T$. 

Suppose that $\vctz \in {\cal M}_k$, i.e., $\vctz$ consists of at most $k$ non-zero entries that each appears in a distinct block of size $n$ (there are a total of $K$ blocks).
First, Lemma~\ref{lma:norm_preservation} shows that the expectation of the norm of $\mtx{W}^T \vctz$ is preserved.
\begin{lemma} \label{lma:norm_preservation}
\[
	\E( \| \mtx{W}^T \vctz \|^2_2 ) = \| \vctz \|^2_2
\]
\end{lemma}
\begin{proof}
Note that each entry of $\mtx{W}^T$ is either zero or Gaussian random variable $w \sim N(0,1)$ before scaling.
Therefore, it is obvious that $\E( \mtx{W} \mtx{W}^T) = \mtx{I}$ since each row of $\mtx{W}$ satisfies $\E \left( \left( \vct{w}_{i_1,m}^{j_1} \right)^T \left( \vct{w}_{i_2,m}^{j_2} \right) \right) = 0$ if $j_1 \neq j_2$ or $i_1 \neq i_2$ for any $m = 1 \dots M$, and we normalized the random variables so that $\E \left( \left\| \left[ \left( \vct{w}_{i,1}^j \right)^T, \ldots, \left( \vct{w}_{i,M}^j \right)^T \right] \right\|_2 \right) = 1$ for all $i,j$'s.
Finally, we have
\begin{align*}
	\E \left( \| \mtx{W}^T \vctz \|_2^2 \right) &= \E \left( \vctz^T \mtx{W} \mtx{W}^T \vctz \right)
		= \vctz^T \E \left( \mtx{W} \mtx{W}^T \right) \vctz
		= \vctz^T \vctz = \| \vctz \|_2^2.
\end{align*}
\end{proof}

Let $\vct{y} = \mtx{W}^T \vctz$.
We aim to show that the square norm of the random variable $\| \vct{y} \|_2^2$ concentrates tightly about its mean; i.e., with exceedingly low probability
\[
	\Big| \|\vct{y}\|_2^2 -  \| \vctz \|^2_2 \Big| > \delta \| \vctz \|^2_2.
\]
To do so, we need several properties of sub-Gaussian and sub-exponential random variables.
A mean-zero \emph{sub-Gaussian} random variable $Z$ has a moment generating function that satisfies
\[
	\E( \exp (tZ) ) \leq \exp(t^2 C^2)
\]
for all $t \in \R$ and some constant $C$.
The \emph{sub-Gaussian norm} of $Z$, denoted $\| Z \|_{\psi_2}$ is
\[
	\| Z \|_{\psi_2} = \sup_{p \geq 1} \frac{1}{\sqrt p} \Big( \E |Z|^p \Big)^{1/p}.
\]
If $Z \sim N(0,\sigma^2)$, then $\|Z\|_{\psi_2} \leq c \sigma$ where $c$ is a positive constant (following Definition 5.7 in \citet{Vershynin:2010vk}).

A \emph{sub-exponential} random variable $X$ satisfies\footnote{There are two other equivalent properties. See~\citet{Vershynin:2010vk} for details.}
\[
	\Prob \Big( |X| > t \Big) \leq \exp(1 - t/C)
\]
for all $t \geq 0$.

 Let $\vct{y}_i$ denote the $i$th entry of the vector $\vct{y} = \mtx{W}^T \vctz$.
 We can write
\[
		\vct{y}_i = \sum_{j=1}^{Kn} \mtx{W}_{i,j} \vctz_j
\]
and observe that $\vct{y}_i$ is a linear combination of i.i.d. sub-Gaussian random variables (or it is identically equal to 0) and, as such, is itself a sub-Gaussian random variable with zero mean and sub-Gaussian norm $\|\vct{y}_i\|_{\psi_2} \leq C/\sqrt{M\ell} \|w\|_{\psi_2} \| \vctz \|_2$ (see~\citet{Vershynin:2010vk}, Lemma 5.9).

The structure of the random matrix and how many non-zero entries are in row $i$ of $\mtx{W}$ do enter the more refined bound on the sub-Gaussian norm of $\|\vct{y}_i\|_{\psi_2} $ (again, see~\citet{Vershynin:2010vk}, Lemma 5.9 for details) but we ignore such details for this estimate as they are not necessary for the next estimate.

To obtain a concentration bound for $\|\vct{y}_i\|^2_2$, we recall from~\citet{Park:2011iw} and \citet{Vershynin:2010vk} that the sum of squares of sub-Gaussian random variables tightly concentrate.
\begin{theorem} \label{thm:modelRIP-sub}
	Let $Y_1,\ldots, Y_{MD}$ be independent sub-Gaussian random variables with sub-Gaussian norms $\|Y_i\|_{\psi_2}$ for all $i = 1, \ldots, MD$.
	Let $T = \max_i \| Y_i \|_{\psi_2}$.
	For every $t \geq 0$ and every $\vct{a} \in \R^{MD}$ and a positive constant $C$, 
	\[
		\Prob \Bigg( \Big| \sum_{i=1}^{MD} \vct{a}_i(Y_i^2 - \E Y_i^2) \Big| \geq t \Bigg) \leq 
			2 \exp\Bigg( -C \min\Big( \frac{t^2}{T^2 \|\vct{a}\|_2^2}, \frac{t}{T \|\vct{a}\|_\infty} \Big)  \Bigg).
	\]
\end{theorem}
We note that although some entries $\vct{y}_i$ may be identically zero, depending on the sparsity pattern of $\vctz$, not all entries are.
Let us define $\tilde{\vct{y}}_i = \frac{\vct{y}_i}{\|\vct{y}_i\|_{\psi_2}}$ so that $\|\tilde{\vct{y}}_i\|_{\psi_2} = 1$.

From Lemma~\ref{lma:norm_preservation} and the relation $\vct{y} = \mtx{W}^T \vct{z}$, we have 
\[
	\Prob\Big( \Big| \|\vct{y}\|_2^2 - \|\vctz\|_2^2 \Big| > \delta \|\vctz\|_2^2 \Big) = 
	\Prob\Bigg( \Big|\sum_{i=1}^{MD} \|\vct{y}_i\|^2_{\psi_2} (\tilde{\vct{y}}^2_i - \E \tilde{\vct{y}}^2_i  ) \Big| > \delta \|\vctz\|_2^2 \Bigg).
\]
See Proposition 5.16 in \citet{Vershynin:2010vk} for the proof of Theorem~\ref{thm:modelRIP-sub}.
We apply Theorem~\ref{thm:modelRIP-sub} to the sub-Gaussian random variables $\tilde{\vct{y}}_i$ with the weights $\| \vct{y}_i \|^2_{\psi_2} $.
We have
\[
	\|\vct{a}\|_2^2 = \sum_{i=1}^{MD} \|\vct{y}_i \|^4_{\psi_2} \leq
	\frac{C D \|w\|_{\psi_2}^4 \|\vctz\|_2^4}{M \ell^2} \quad\text{and}\quad
	\|\vct{a}\|_\infty \leq \frac{C \|w\|_{\psi_2}^2 \|\vctz\|_2^2}{M \ell}.
\]

If we set $T=1$, $t=\delta \|\vctz\|_2^2$, and use the above estimates for the norms of $\vct{a}$, we have
\begin{equation}
\label{eqn:CoM}
	\Prob\Big( \Big| \|\vct{y}\|_2^2 - \|\vctz\|_2^2 \Big| > \delta \|\vctz\|_2^2 \Big) \leq
	2 \exp\Bigg( -C \min\Big( 
	   \frac{\delta^2 M \ell^2} {D \|w\|_{\psi_2}^4} ,
	   \frac{\delta M \ell} { \|w\|_{\psi_2}^2 }
	 \Big)  \Bigg).
\end{equation}

Finally, we use the concentration of measure result in a crude union bound to bound the failure probability over all vectors $\vctz \in {\cal M}_k$.
We take $n^k\binom{K}{k} \approx (nK)^k$ and $\epsilon$ for a desired constant failure probability.
Using the smaller term in \eqref{eqn:CoM}, (note that $\delta < 1$, $\ell / D < 1$, and $\|w\|_{\psi_2} \leq 1$) we have
\[
	\exp\Big(-C \frac{M \ell^2 \delta^2}{D \|w\|_{\psi_2}^2} \Big) \exp\Big(k (\log(K) + \log(n))\Big) \leq \exp(\log(\epsilon))
\]
which implies
\[
	\frac{M \ell^2}{D} \geq \frac{\|w\|_{\psi_2}^2}{\delta^2} \Bigg( k (\log(K) 
	   + \log(n)) - \log(\epsilon) \Bigg) = \frac{C}{\delta^2}\Big( k (\log(K) + \log(n))
	    - \log(\epsilon)\Big).
\]
Therefore, if design our matrix $\mtx{W}$ as described and with the parameter relationship as above, the matrix $\mtx{W}^T$ satisfies the model-RIP for ${\cal M}_k$ and parameter $\delta$ with probability $1 - \epsilon$. 
\end{proof}

Let us discuss the relationship amongst the parameters in our result.
First, if we have only one channel $M = 1$ and the filter length $\ell = D$
; namely,
\[
	D \geq \frac{C}{\delta^2} \left( k (\log(K) + \log(n)) - \log(\epsilon) \right).
\]
If $\ell < D$ (i.e., the filters are much shorter than the length of the input signal as in a CNN), then we can compensate by adding more channels; i.e., the filter length $\ell$ needs to be larger than $\sqrt{D}$, or, if add more channels, $\sqrt{D / M}$.

 \section{Mathematical Analysis: Reconstruction Bounds}
\label{sec:reconstruct}

The consequences of having the model-RIP are two-fold.
The first is that if we assume that an input image is the structured sparse linear combination of filters, $\vctx = \mtx{W}^T \vctz$ where $\vctz \in {\cal M}_k$ and $\mtx{W}^T$ satisfies the model-RIP, then we know an upper and lower bound on the norm of $\vctx$ in terms of the norm of its sparse coefficients, $\| \vctx \|_2 \leq (1 \pm \delta) \| \vctz \|_2$.
Additionally,
\[
	\| \vctz \|_2 \leq \frac{1}{\sqrt{1 - \delta}} \| \vctx\|_2.
\]

More importantly, when we calculate the hidden units of $\vctx$,
\[
		\vcthid = \mtx{W} \vctx = \mtx{W} \mtx{W}^T \vctz,
\]
then we can see that the computation of $\vcthid$ is nothing other than the first step of a reconstruction algorithm analogous to that of model-based compressed sensing.
As a result, we have a bound on the error between $\vcthid$ and $\vctz$ and we see that we can analyze the approximation properties of a feedfoward CNN and its linear reconstruction algorithm.
In particular, we can conclude that a feedforward CNN and a linear reconstruction algorithm provide a good approximation to the original input image.

\begin{reftheorem}{thm:reconstruction}{(Restated)}
	We assume that $\mtx{W}^T$ satisfies the ${\cal M}_k^2$-RIP with constant $\delta_k \leq \delta_{2k} < 1$.
	If we use $\mtx{W}$ in a single layer CNN both to compute the hidden units $\vcthatz$ and to reconstruct the input $\vctx$ from these hidden units as $\vcthatx$ so that $\hat{\vctx}=\mtx{W}^T\mathbb{M}(\mtx{W}\vctx,k)$, the error in our reconstruction is  
	\[
		\|\vcthatx - \vctx\|_2 \leq \frac{5 \delta_{2k}}{1 - \delta_{k}}
					\frac{\sqrt{1 + \delta_{2k}}}{\sqrt{1 - \delta_{2k}}}  \|\vctx\|_2.
	\]
\end{reftheorem}
\begin{proof}
To show this result, we recall the following lemmas from~\citet{Baraniuk:2010hg} and rephrase them in the setting of a feedforward CNN.
Note that Lemma~\ref{lma:rip-bound} and ~\ref{lem:contaminate} are the same as Lemma~1 and 2 in \citet{Baraniuk:2010hg}, respectively.

\begin{lemma}
	Suppose $\mtx{W}^T$ has ${\cal M}_k$-RIP with constant $\delta_k$.
	Let $\Omega$ be a support corresponding to a subspace in ${\cal M}_k$.
	Then we have the following bounds:
	\begin{align}
		\| \mtx{W}_\Omega \vctx \|_2 & \leq \sqrt{1 + \delta_k} \|\vctx\|_2 
		\label{eq:recon} \\
		\| \mtx{W}_\Omega \mtx{W}^T_\Omega \vctz \|_2 & \leq (1 + \delta_k) \|\vctz\|_2 
		\label{eq:idUB} \\
		\| \mtx{W}_\Omega \mtx{W}^T_\Omega \vctz \|_2 & \geq (1 - \delta_k) \|\vctz\|_2 
		\label{eq:idLB}
	\end{align}
	\label{lma:rip-bound}
\end{lemma}

\begin{lemma}
\label{lem:contaminate}
	Suppose that $\mtx{W}^T$ has ${\cal M}_k^2$-RIP with constant $\delta_{2k}$.
	Let $\Omega$ be a support corresponding to a subspace of ${\cal M}_k$ and suppose that $\vctz \in {\cal M}_k$ (not necessarily supported on $\Omega$).
	Then
	\[
		\| \mtx{W}_\Omega \mtx{W}^T \vctz|_{\Omega^c} \|_2 \leq \delta_{2k} \| \vctz|_{\Omega^c} \|_2.
	\]
\end{lemma}

Let $\Pi$ denote the support of the ${\cal M}_k$ sparse vector $\vctz$.
Set $\vcthid = \mtx{W} \vctx$ and set $\vcthatz$ to be the result of max pooling applied to the vector $\vcthid$, or the best fit (with respect to the $\ell_2$ norm) to $\vcthid$ in the model ${\cal M}_k$.
Let $\Omega$ denote the support set of $\vcthatz \in {\cal M}_k$.
For simplicity, we assume $|\Pi| = k = |\Omega|$. 

\begin{lemma}[Identification]
	The support set, $\Omega$, of the switch units captures a significant fraction of the total energy in the coefficient vector $\vctz$
	\[
		\| \vctz|_{\Omega^c} \|_2 \leq \frac{2\delta_{2k}}{1 - \delta_{k}} \| \vctz \|_2.
	\]
\end{lemma}
\begin{proof}
	Let $\vcthid_\Omega$ and $\vcthid_\Pi$ be the vector $\vcthid$ restricted to the support sets $\Omega$ and $\Pi$, respectively.
	Since both are support sets for ${\cal M}_k$ and since $\Omega$ is the best support set for $\vcthid$,
	\[
		\|\vcthid - \vcthid_\Omega \|_2 \leq \| \vcthid - \vcthid_\Pi \|_2,
	\]
	and, after several calculations, which is identical to those in the proof of Lemma 3 in \citet{Baraniuk:2010hg},
	we have
	\[
		\| \vcthid|_{\Omega \setminus \Pi} \|_2^2 \geq \| \vcthid|_{\Pi \setminus \Omega} \|_2^2.
	\]
	Using Lemma~\ref{lem:contaminate} and the size $|(\Omega \setminus \Pi) \bigcup \Pi | \leq 2k$, we have
	\[
		\| \vcthid_{\Omega \setminus \Pi} \|_2 = \| \mtx{W}_{\Omega \setminus \Pi} \mtx{W}^T \vctz\|_2
		\leq \delta_{2k} \|\vctz\|_2.
	\]
	
    Using \eqref{eq:idLB} and Lemma~\ref{lem:contaminate},
	we can bound the other side of the inequality as
	\begin{align*}
		\| \vcthid_{\Pi \setminus \Omega} \|_2
		  &= \| \mtx{W}_{\Pi \setminus \Omega} \mtx{W}^T \vctz \|_2 \\
		  &\geq \| \mtx{W}_{\Pi \setminus \Omega} (\mtx{W}^T \vctz|_{\Pi \setminus \Omega} ) \|_2 - 
		  \| \mtx{W}_{\Pi \setminus \Omega} (\mtx{W}^T \vctz|_{\Omega} ) \|_2 \\
		  &\geq (1 - \delta_{k}) \|  \vctz|_{\Pi \setminus \Omega} \|_2 - 
		      \delta_{2k} \|\vctz|_{\Omega} \|_2 \\
		  &\geq (1 - \delta_{k}) \|  \vctz|_{\Pi \setminus \Omega} \|_2 - 
		      \delta_{2k} \|\vctz \|_2.
	\end{align*}
	
	Since the support of $\vctz$ is the set $\Pi$, $\Pi \setminus \Omega = \Omega^c$ for $\vctz$, so we can conclude that
	\[
		\delta_{2k} \|\vctz\|_2 \geq (1 - \delta_{k}) \|\vctz|_{\Omega^c}\|_2 - 
		    \delta_{2k} \|\vctz\|_2,
	\]
	and with some rearrangement, we have
	\[
		\| \vctz|_{\Omega^c}\|_2 \leq \frac{2 \delta_{2k}}{1 - \delta_{k}} \|\vctz\|_2.
	\]
\end{proof}

To set the value of $\vcthatz$ on its support set $\Omega$, we simply set $\vcthatz = \vcthid|_\Omega$ and $\vcthatz|_{\Omega^c} = 0$.
Then
\begin{lemma}[Estimation]
\label{lem:estimate}
	\[
		\| \vctz - \vcthatz \|_2 \leq \frac{5 \delta_{2k}}{1 - \delta_{k}} \|\vctz\|_2
	\]
\end{lemma}
\begin{proof}
	First, note that $\| \mtx{I} - \mtx{W}_{\Omega} \mtx{W}_{\Omega}^T \|_2 \leq \max\{ (1+\delta_{k})-1, 1-(1-\delta_{k}) \} = \delta_{k}$ since 
	\[
		(1-\delta_k) \leq \sup_{\| \vctz \| \neq 0} \frac{\| \mtx{W}_{\Omega}^T \vctz \|_2^2}{\| \vctz \|_2^2} \left( = \sigma_{\max}^2 ( \mtx{W}_{\Omega}^T ) = \sigma_{\max} ( \mtx{W}_{\Omega} \mtx{W}_{\Omega}^T ) \right) \leq (1+\delta_k),
	\]
	where $\sigma_{\max}$ is the maximum singular value.
	Therefore,
	\begin{align*}
		\| \vctz - \vcthatz \|_2 &\leq \|\vctz|_{\Omega^c} \|_2 + 
		       \| \vctz|_\Omega - \vcthatz|_\Omega \|_2 \\
		&= \|\vctz|_{\Omega^c} \|_2 + \| \vctz|_\Omega - 
		  \mtx{W}_\Omega (\mtx{W}^T \vctz|_{\Omega} + \mtx{W}^T\vctz|_{\Omega^c}) \|_2 \\
		&\leq \| \vctz |_{\Omega^c} \|_2 + \| (\mtx{I} - \mtx{W}_{\Omega} \mtx{W}_{\Omega}^T) \vctz |_{\Omega} \|_2 + \| \mtx{W}_{\Omega} \mtx{W}^T \vctz |_{\Omega^c} \|_2 \\
		&\leq \| \vctz |_{\Omega^c} \|_2 + \| \mtx{I} - \mtx{W}_{\Omega} \mtx{W}_{\Omega}^T \|_2 \| \vctz |_{\Omega} \|_2 + \delta_{2k} \| \vctz |_{\Omega^c} \|_2 \\
		&\leq \|\vctz|_{\Omega^c} \|_2 + \delta_k \| \vctz|_\Omega\|_2 + \delta_{2k} \|\vctz|_{\Omega^c} \|_2 \\
		&\leq \Big( (1 + \delta_{2k}) \frac{2\delta_{2k}}{1 - \delta_{k}} + \delta_k \Big) 
		    \|\vctz\|_2 \\
		&\leq \frac{5 \delta_{2k}}{1 - \delta_{k}} \|\vctz\|_2.
	\end{align*}
\end{proof}

Finally, if we use the autoencoder formulation to reconstruct the original image $\vctx$ by setting $\vcthatx = \mtx{W}^T \vcthatz$, we can estimate the reconstruction error.
We note that $\vcthatz$ is ${\cal M}_k$-sparse by construction and remind the reader that $\mtx{W}^T$ satisfies ${\cal M}_k^2$-model-RIP with constants $\delta_{k} \leq \delta_{2k} \ll 1$.
Then, using Lemma~\ref{lem:estimate} as well as the ${\cal M}_k^2$-sparse properties of $\mtx{W}^T$, 
	\begin{align*}
		\| \vctx - \vcthatx\|_2 &= \| \mtx{W}^T(\vctz - \vcthatz) \|_2 
		    \leq \sqrt{1 + \delta_{2k}} \| \vctz - \vcthatz \|_2 \\ 
		&\leq \frac{5 \delta_{2k}}{1 - \delta_{k}} \sqrt{1 + \delta_{2k}} \| \vctz \|_2 \\
		&\leq \frac{5 \delta_{2k}}{1 - \delta_{k}} \frac{\sqrt{1 + \delta_{2k}}}{\sqrt{1 - \delta_{2k}}}  \|\vctx\|_2.
	\end{align*}
This proves that a feedforward CNN with a linear reconstruction algorithm is an approximate autoencoder and bounds the reconstruction error of the input image in terms of the geometric properties of the filters.
\end{proof} 
\section{More Experimental Results}

\subsection{More Details on Evaluation of CNNs with Gaussian Random Filters}
\label{sec:cifar-random-filters-detailed}
In this section, we provide more details on the network architectures that we used in Table~\ref{tab:cifar-random-filters}. 
For the network architecture,
we add a batch normalization layer together with a learnable scale and bias before the activation so that we do not need to tune the scale of the filters. 
The filter weights of the intermediate layers in the CNNs are not trained after random initialization. 
On top of the network, we use an optional average pooling layer to reduce the feature map size to $4 \times 4$ and a dropout layer for better regularization before feeding the feature to a learnable soft-max classifier for image classification.
We describe the best performing architectures for all cases in Table~\ref{tab:cifar-random-filters-detailed}.

\begin{table}[H]

\begin{center}
\small{}
\begin{tabular}{c||c|c|c|c}
\hline 
Method & \#Layers & 1 layers  & 2 layers  & 3 layers \tabularnewline
\hline 
\multirow{3}{*}{Random filters } & {\footnotesize{}Best } & \multirow{2}{*}{{\footnotesize{}(2048)5c-2p$_{\textrm{max}}$-4p$_{\textrm{ave}}$ }} & {\footnotesize{}(2048)3c-2p$_{\textrm{max}}$- } & {\footnotesize{}(2048)3c-2p$_{\textrm{max}}$-(2048)3c-}\tabularnewline
 & {\footnotesize{}param.} &  & {\footnotesize{}(2048)3c-2p$_{\textrm{max}}$-2p$_{\textrm{ave}}$ } & {\footnotesize{}2p$_{\textrm{max}}$-(1024)3c-2p$_{\textrm{max}}$ }\tabularnewline
\cline{2-5} 
 & Accuracy & 66.5\%  & 74.6\%  & 74.8\%\tabularnewline
\hline 
\multirow{3}{*}{Learned filters} & {\footnotesize{}Best } & \multirow{2}{*}{{\footnotesize{}(1024)5c-2p$_{\textrm{max}}$-4p$_{\textrm{ave}}$ }} & {\footnotesize{}(1024)3c-2p$_{\textrm{max}}$- } & {\footnotesize{}(1024)3c-2p$_{\textrm{max}}$-(1024)3c-}\tabularnewline
 & {\footnotesize{}param.} &  & {\footnotesize{}(1024)3c-2p$_{\textrm{max}}$-2p$_{\textrm{ave}}$ } & {\footnotesize{}2p$_{\textrm{max}}$-(1024)3c-2p$_{\textrm{max}}$ }\tabularnewline
\cline{2-5} 
 & Accuracy & 68.1\%  & 83.3\%  & 89.3\% \tabularnewline
\hline 
\end{tabular}

\end{center}

\caption{Best-performing architecture and classification accuracy of random CNNs on CIFAR-10. ``($\mathtt{[n]}$)$\mathtt{[k]}$c'' denotes a convolution layer with a stride $1$, a kernel size $\mathtt{[k]}$ and $\mathtt{[n]}$ output channels, ``$\mathtt{[k]}$p$_{\textrm{max}}$'' denotes a max pooling layer with a kernel size $\mathtt{[k]}$ and a stride $\mathtt{[k]}$, and ``$\mathtt{[k]}$p$_{\textrm{ave}}$'' denotes a average pooling layer. A typical layer consists of four operations, namely convolution, ReLU, batch normalization, and max pooling.}
\label{tab:cifar-random-filters-detailed}
\end{table}

\subsection{Layer-wise Coherence and Sparsity for AlexNet} \label{sec:alexnet_sparsity}
We present coherence (see Table~\ref{tab:alexnet_coherence}) and sparsity level (see Table~\ref{tab:alexnet_sparsity}) for each layer in AlexNet.
\begin{table}[H]
\small{}
\begin{center}
\begin{tabular}{ c || c | c | c | c | c }
\hline
layer & 1 & 2 & 3 & 4 & 5\\
\hline
coherence of learned filters & 0.9172 & 0.6643 & 0.6200 & 0.6382 & 0.3390 \\
\hline
coherence of random filters & 0.1996 & 0.1263 & 0.0929 & 0.1073 & 0.1026 \\
\hline
\end{tabular}

\end{center}

\caption{Comparison of coherence between learned filters in each layer of AlexNet and Gaussian random filters with corresponding sizes.}
\label{tab:alexnet_coherence}
\end{table}

\begin{table}[H]
\begin{center}
\begin{tabular}{c||c|c|c|c|c|c|c|c}
\hline
layer & conv1 & pool1 & conv2 & pool2 & conv3 & conv4 & conv5 & pool5 \\
\hline
\% of non-zeros & 49.41 & 87.79 & 18.97 & 44.13 & 31.08 & 30.95 & 9.78 & 28.15 \\
\hline
\end{tabular}
\end{center}
\caption{Layer-wise sparsity of AlexNet on ILSVRC-2012 validation set.}
\label{tab:alexnet_sparsity}
\end{table}

\subsection{Visualization of Image Reconstruction for VGGNet}
\label{sec:recon-vgg}
In Figure~\ref{fig:RIP_vgg_recon_images}, we show reconstructed images from each layer using different reconstruction methods via a pretrained decoding network.
\begin{figure}[h!]
\centering
\includegraphics[height=\textheight]{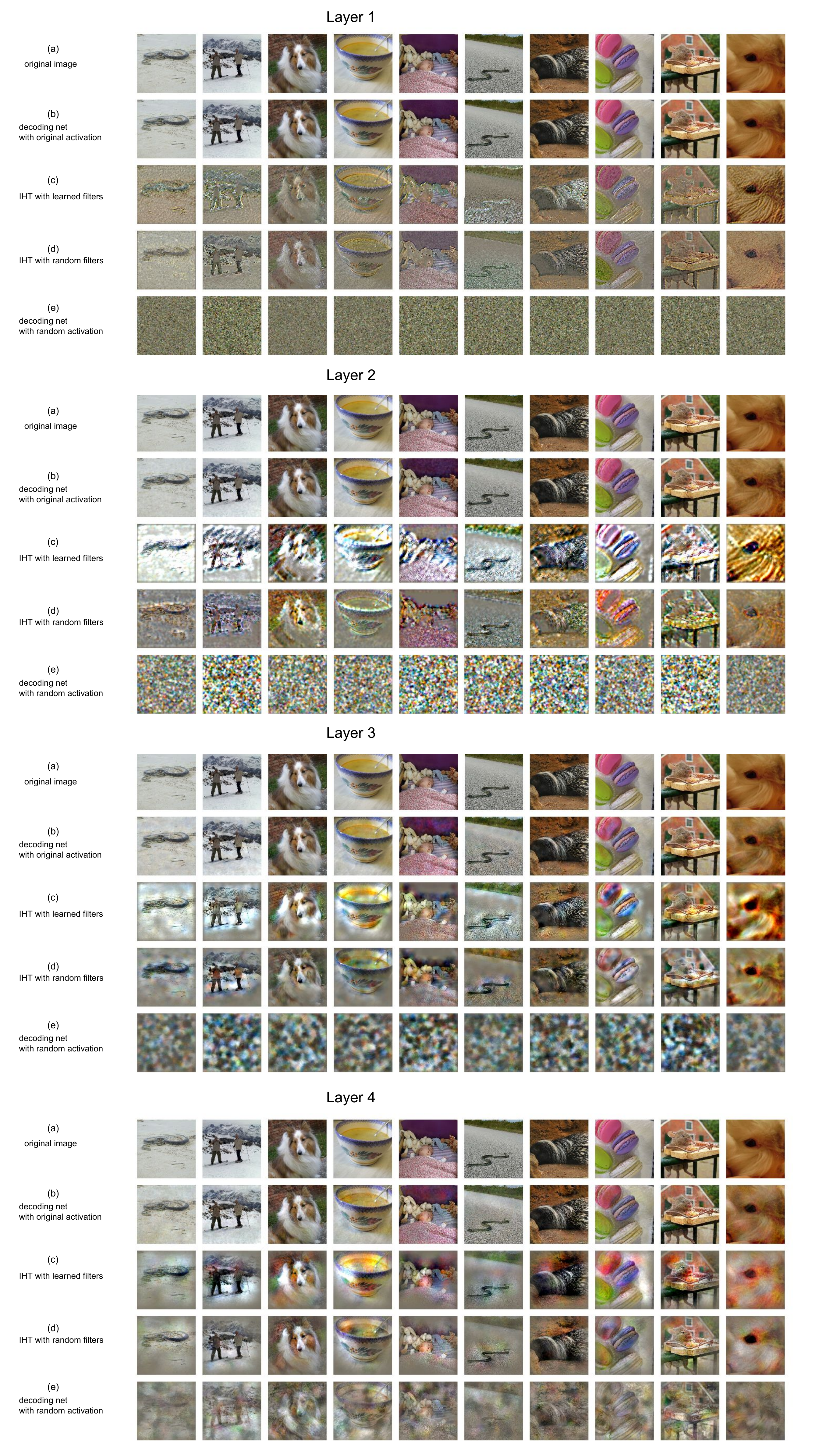}
\caption{Visualization of images reconstructed by a pretrained decoding network with VGGNet's pool$(4)$ activation reconstructed using different methods: (a)~original image, (b)~output of the $5$-layer decoding network with original activation, (c) output of the decoding net with reconstructed activation by IHT with learned filters, (d) output of the decoding net with reconstructed activation by IHT with Gaussian random filters, (e) output of the decoding net with Gaussian random activation.}
\label{fig:RIP_vgg_recon_images}
\end{figure}

\end{document}